%% file: main.tex
\documentclass[conference]{IEEEtran}
\IEEEoverridecommandlockouts
\usepackage{cite}
\usepackage{amsthm,amsmath,amssymb,amsfonts}
\usepackage{graphicx}
\usepackage{textcomp}
\usepackage{xcolor}

\usepackage{bm}
\usepackage{amsmath, amsfonts,amssymb}
\usepackage{algorithm}
\usepackage{algpseudocode}
\usepackage{subcaption}
\usepackage{multirow}
\usepackage{url}
\usepackage{color}
\usepackage{float}
\usepackage{enumitem}
\newtheorem{theorem}{Theorem}

\newtheorem{corollary}[theorem]{Corollary}
\def\BibTeX{{\rm B\kern-.05em{\sc i\kern-.025em b}\kern-.08em
    T\kern-.1667em\lower.7ex\hbox{E}\kern-.125emX}}
\begin{document}
\def\UrlBreaks{\do\/\do-}

\title{
ACE -- An Anomaly Contribution Explainer \\ for  Cyber-Security Applications
}

\makeatletter
\newcommand{\linebreakand}{%
  \end{@IEEEauthorhalign}
  \hfill\mbox{}\par
  \mbox{}\hfill\begin{@IEEEauthorhalign}
}
\makeatother

\author{\IEEEauthorblockN{Xiao Zhang\thanks{Work done during an internship at \textbf{HPE Software} (now \textbf{Micro Focus}).}}
\IEEEauthorblockA{\textit{Department of CS, Purdue University} \\
USA \\
zhang923@purdue.edu}
\and
\IEEEauthorblockN{Manish Marwah}
\IEEEauthorblockA{\textit{Micro Focus} \\
USA \\
manish.marwah@microfocus.com}
\and
\IEEEauthorblockN{I-ta Lee}
\IEEEauthorblockA{\textit{Department of CS, Purdue University} \\
USA \\
lee2226@cs.purdue.edu}
\linebreakand
\IEEEauthorblockN{Martin Arlitt}
\IEEEauthorblockA{\textit{Micro Focus} \\
Canada \\
martin.arlitt@microfocus.com}
\and
\IEEEauthorblockN{Dan Goldwasser}
\IEEEauthorblockA{\textit{Department of CS, Purdue University} \\
USA \\
dgoldwas@cs.purdue.edu}
}

\maketitle

\begin{abstract}
In this paper we introduce \textit{Anomaly Contribution Explainer} or \textit{ACE}, a tool to explain security anomaly detection models in terms of the model features through a regression framework, and its variant, \textit{ACE-KL}, which highlights the important anomaly contributors.  ACE and ACE-KL provide insights in diagnosing which attributes significantly contribute to an anomaly by building a specialized linear model to locally approximate the anomaly score that a black-box model generates. We conducted experiments with these anomaly detection models to detect security anomalies on both synthetic data and real data.  In particular, we evaluate performance on three public data sets: CERT insider threat, netflow logs, and Android malware. The experimental results are encouraging: our methods consistently identify the correct contributing feature in the synthetic data where ground truth is available; similarly, for real data sets, our methods point a security analyst in the direction of the underlying causes of an anomaly, including in one case leading to the discovery of previously overlooked network scanning activity. We have made our source code publicly available.
\end{abstract}

\begin{IEEEkeywords}
anomaly detection, model explanation, model interpretability, cyber-security
\end{IEEEkeywords}

\input{intro}
\input{relwork}
\input{methods}

\input{exp}
\input{concl}

\bibliographystyle{IEEEtran}

\bibliography{mybib}

\input{appendix}

\end{document}

%% file: intro.tex
\section{Introduction}

Cyber-security is a key concern for both private and public
organizations, given the high cost of security compromises and attacks; 
malicious cyber-activity cost the U.S. economy between \$57 billion
and \$109 billion in 2016 \cite{whitehousereport}. As a result,
spending on security research and development, and security products
and services to detect and combat cyber-attacks has been increasing
\cite{forbesreport}.

Organizations produce large amounts of network, host and application
data that can be used to gain insights into cyber-security threats,
misconfigurations, and network operations.  While security domain
experts can manually sift through some amount of data to spot attacks
and understand them, it is virtually impossible to do so at scale,
considering that even a medium sized enterprise can produce terabytes
of data in a few hours. Thus there is a need to automate the process
of detecting security threats and attacks, which can more generally be
referred to as security anomalies.

Major approaches to detect such anomalies fall into two
broad categories: human expert driven (mostly rules-based) and machine
learning based (mostly unsupervised) \cite{Veeramachaneni2016}.  The
first approach involves codifying domain expertise into rules, for
example, if the number of login attempts exceeds a threshold, or more
than a threshold number of bytes are transferred during the night, and so
on.  While rules formulated by security experts are useful, they are
ineffective against new (zero-day) and evolving attacks; furthermore,
they are brittle and difficult to maintain.  On the other hand,
enabled by the vast amounts of data collected in modern enterprises,
machine learning based approaches have become the preferred choice for
detecting security anomalies. 

The machine learning models to detect security anomalies typically
output a severity or anomaly score; this score allows ranking and
prioritization of the anomalies. A security analyst can then further
investigate these anomalies to understand their root causes, if they
are true positives, and if any remedial action is required.  However,
anomaly detectors typically do not provide any assistance in this
process. In fact, any direction or pointers in terms of features, or
groups of features, responsible for a high anomaly score would allow
prioritization of causes to look at first and thus save an analyst's
time and effort; this would help even though the information may not
directly reveal the root cause of an anomaly.  For example, based on
the contributions an anomaly detector assigns to features related to
external traffic volume, number of ports active, number of external
hosts, etc, analysts would decide the course of their investigation into
the underlying causes of that particular anomaly.

However, most anomaly detection models are black-boxes that output an
anomaly score without any associated explanation or reasoning. In
fact, there is an inverse relationship between building complex models
that can make accurate predictions and explaining these predictions in
a human-interpretable way. For example, explaining the predictions of
simpler models, such as linear regression, logistic regression or
decision trees, is considerably easier compared to complex models such
as random forests or deep neural networks, which build complex
non-linear relationships between the input and the predicted output.

As a result, when models that can explain their output are needed, as
is often the case, for example, in medical diagnosis (a doctor needs
to provide a detailed explanation of the diagnosis to the patients
\cite{Caruana2015}), or credit card application (an explanation of why
or why not a particular application is approved is usually required
\cite{Shi2012}), simpler models are preferred. 
However, interpretability comes at a cost since in most instances
complex models tend to have higher accuracy. Therefore, there is an
unavoidable trade-off between model interpretability and model
accuracy. Recently deep learning models are being successfully
used for cyber-security applications \cite{Tuor2017,
  yousefi2017autoencoder, berman2019survey, cui2018detection}. In fact,
a part of the focus of a recently organized workshop is the application of
deep learning to security \cite{dls}.

In this paper, we focus on explaining the outputs of complex models in
the cyber-security anomaly detection domain, where outputs are usually
anomaly scores.  We propose ACE -- \textit{A}nomaly \textit{C}ontribution 
\textit{E}xplainer, to bridge the gap between the predictions 
provided by an anomaly detection model and
the interpretation required to support human intervention in realistic
applications. Specifically, ACE provides explanations, in
terms of the features' contributions, by building a specialized linear 
model to locally approximate the anomaly score that a black-box anomaly 
detection model generates.
These explanations aid a security analyst to quickly
diagnose the reported anomaly. Our source code is publicly available\footnote{Source code available at
  \url{https://github.com/cosmozhang/ACE-KL}}.

Our key contributions are:
\begin{itemize}
\item We design and implement two methods, ACE and ACE-KL, for explaining scores of individual anomalies detected by black-box models in terms of feature contributions.
\item We validate our methods on three data sets: 1) a synthetically generated insider threat data set; 2) a real world netflow data set; and 3) a real world Android malware data set. In all of these cases, the results are encouraging and improved upon a recent work \cite{Ribeiro2016a} as a baseline.
\end{itemize}

The high-level overview of our approach is shown in Figure~\ref{fig:overview}, and we focus on the meta-model approximation to explain the score of a particular anomaly.

\begin{figure}[!htb]
\centering
\includegraphics[width=0.48\textwidth]{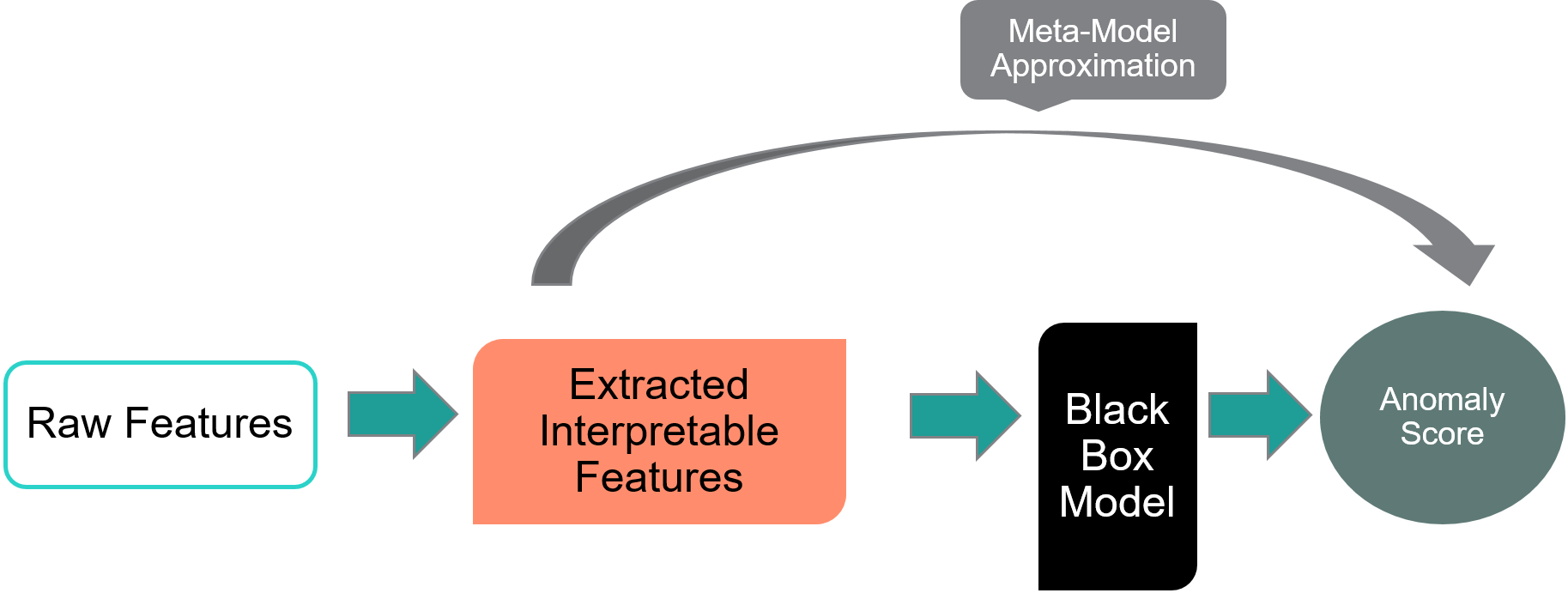}
\caption{Overview of the anomaly detection and interpretation workflow for the ACE or ACE-KL meta-model.}
\label{fig:overview}
\end{figure}

%% file: relwork.tex
\section{Related Work}
While model interpretability and explanation have a long history \cite{biranSurvey2017,doshi2017,Lipton2018}, the recent success and rise in popularity of complex machine learning models (such as deep neural networks) has led to a surge of interest in model interpretability, as these complex models offer no explanation of their output predictions. 
Given the extensive literature on this subject, we only discuss work most related to ours; Guidotti et al.\cite{GuidottiSurvey2018} provide a comprehensive survey on explainability.   

Methods for generating explanations from complex models fall into two main categories: 1) model-specific ~\cite{suermondt1992,feraud2002methodology,robnik2011,landecker2013,martens2008}, which exploit a model's internal structure and as a result only work for that specific model type; and 2) model-agnostic \cite{robnik2008,kononenko2013,baehrens2010} or black-box methods, which do not depend on the underlying model type. Our work belongs to the second category. 

Several model-agnostic methods investigate the sensitivity of output with respect to the inputs to explain the output: An early attempt called \textit{ExplainD} used additive models to weight the importance of features with a graphical explanation  \cite{Poulin2006}. In Strumbelj et al.'s 
work \cite{Strumbelj2010}, the authors exploited notions from coalitional game theory to explain the contribution of value of different individual features. \textit{LIME}\cite{Ribeiro2016a}, designed for classification problems, was built to explain a new data point after a black-box model is trained. \textit{MFI} \cite{Vidovic2016} is a non-linear method able to detect features impacting the prediction through their interaction with other features. More recently, \textit{LORE} \cite{Guidotti2018} used a synthetic neighbourhood generated through a genetic algorithm to explain the feature importance and \textit{SHAP} \cite{Lundberg2017} assigns each feature
an importance value for a particular prediction. Our proposed methods belong to this category, and are closest to LIME \cite{Ribeiro2016a}.

Anomaly detection is widely studied and an important topic in data mining. However, explanation of the detected anomalies has received relatively little attention from researchers. For instance, one of the most widely cited surveys on anomaly detection \cite{chandola2009} makes no reference to explainability. Research on anomaly detection and explainability includes: work on anomaly localization \cite{Hara2015,Jiang2011}, which refers to the task of identifying faulty sensors from a population of sensors; feature selection or importance \cite{Hara2017}; estimating model sensitivity \cite{Woodall2003}; and method specific techniques \cite{Hirose2009,Tsuyoshi2009}. 
Despite their advantages, these methods are either tailored too closely for specific anomaly detection methods, or only consider sensitivity of inputs, not their entire contribution, and are not suitable for the security domain where anomalies and methods to detect evolve rapidly. 

%% file: methods.tex
\section{Methods}

\subsection{Problem Statement}

Formally, the model explanation problem in the context of anomaly detection can be stated as follows:

Given 1) a black-box anomaly detection model $f$, an arbitrary function with an input $\bm{x}$ having $M$ features: $x^1,\dots,x^M$, which outputs an anomaly score $As$, that is, 
$
f: \bm{x} \rightarrow As,
$
where $ As \in [0, \infty] $ is a scalar; and 2) a random data point $\bm{x}'$ that produces score $As'$, the goal is to estimate $c^1,\dots, c^M$, the normalized contributions of
each feature of $\bm{x}'$. Note that a $c^i$ may be zero if it does not contribute to an anomaly. 


\subsection{Assumptions and Observations}

We assume the output of the anomaly detector is an anomaly score in the range $[0, \infty]$, with 0 indicating no anomaly, and the score monotonically increasing with the severity of the anomaly. Such a score is widely used in anomaly detectors. Note that if an anomaly detector outputs an anomaly probability, $P_A(X)$, it is easy to convert it to such a score by the transformation: $As = -\log(1- P_A(X))$.

A careful study of existing techniques such as LIME \cite{Ribeiro2016a}  reveals their unsuitability for anomaly explanation.
These methods can only explain the importance of a feature locally, but not the whole contribution of that feature. For example, consider a linear regression scenario, i.e., $ \sum\limits_{i=1}^{M}x^i\cdot w^i = As$, where $ x^i $ is the $i$th feature of the vector (here we also encode the bias term $b$ as $w^M$, and the corresponding $x^M$ is always $ 1 $).  Assume a given feature $ x^a $ is of ``small importance'', determined by its non-significant corresponding weight $ w^a $, (assuming $ w^a $ is a value close to $ 0.01 $). However, if $ x^a $ is extremely large in a new example, for instance, $ 1000 $, and $ As=50 $, the multiplication of $ x^a $ and $ w^a $ still makes a large contribution to the predicted value $ As $. 

This observation has practical implications in anomaly detection problems, especially in security-related problems. For example, when a feature tends to appear in some range of values in training, a trained black-box model will weigh it accordingly. After the well-trained model is deployed, a new attack prototype can evolve focusing on specific attributes, which were neglected at training time, but now takes high attribute values. Even if the anomaly may be detected by a well trained black-box model as it results in high output scores, the underlying reason might escape the security analysts' attention. 


\subsection{Anomaly Contribution Explainer (ACE)}


ACE explains a specific anomaly by computing its feature contribution vector obtained through a local linear approximation of the anomaly score. Using this simple approximation, the real contribution that $i$th feature $ x^i $ makes to form $ As $ is naturally $ x^i\cdot w^i $. However, it is possible some $ x^i\cdot w^i $ are negative. These terms correspond to features that negatively impact an anomaly, and thus cannot be its cause. We want to discard these terms and focus on the features positively contributing to an anomaly. Therefore, we use the ``softplus'' function \cite{Dugas2000}, which is a ``smoothed'' \textit{relu} function to model the contribution of $ x^i\cdot w^i $ towards the entire anomaly. The intuition behind this choice is evident: we calculate the contribution by neglecting the negative components while considering the positive part linearly; this function forces all negative components to $0$ and retains all the positive components linear to their original value; further, the convexity of this function simplifies the computation.

We define $As$ as the anomaly score, calculated by the blackbox model. Further, to normalize all of the contributions towards the anomaly score, by denoting the normalized contribution of feature $ i $ as $ c^i $, we formally define the normalized contribution (``contribution" thereafter) of each feature as 
\par\nobreak{\small
\begin{equation}
c^{i}=\dfrac{\log(1+e^{ x^i\cdot w^i})}{\sum\limits_{j=1}^{M}\log(1+e^{x^i\cdot w^i})}.
\label{cal_contribution}
\end{equation}
}

To approximate a particular anomaly score $As$ generated by a black-box model at a point $\bm{x}$ of interest, we form the loss function with a modified linear regression, by sampling the neighborhood of $\bm{x}$ to obtain $N$ neighbors and obtaining their corresponding $N$ anomaly scores:
\begin{equation}
loss = \dfrac{1}{N}\sum\limits_{j=1}^{N} \pi_{\bm{x}}(\bm{x}_{j})\cdot(\bm{w}^{\intercal}\bm{x}_{j}-As_{j})^{2} + \alpha ||\bm{w}||^{2}_{2}\notag,
\end{equation}
where $ As_{j} $ is the anomaly score generated by a black-box model for the $ j $th neighbor, $ \alpha $ set to be $1$ in this study is the hyper parameter that controls the $ L2 $ norm regularizer, and  $ \pi_{\bm{x}}(\bm{x}_{j}) $ is the weight calculated by a distance kernel for the $ j $th neighbor. 
The parameters are estimated by minimizing the loss function, using the neighbourhood of the original example formed through sampling. Based on the fact that this neighbourhood is close enough to the point of intersection between the surface and the tangent plane, we use this neighbourhood to approximate the tangent plane, which is the linear regression.
We choose the normal distribution $ \mathcal{N}(\bm{x}, 0.01\mathbf{I}) $, where $ \mathbf{I} $ is an identity matrix for continuous features, as the neighborhood region to ensure the samples are close enough to the examined point; and a $\operatorname{Bern} \left(\bm{0.1}\right)$ distribution to flip the value for binary features $\bm{x} \in \{0, 1\}^{M}$ for the same reason. A distance kernel $ \pi_{\bm{x}} $ is used to calculate the distance between the examined point and the neighbors as such: $ \pi_{\bm{x}}(\bm{x}_{j}) = \exp(-D(\bm{x}, \bm{x}_{j})^{2}/\sigma^{2}) $, where $ D(\bm{x}, \bm{x}_{j}) $ is the distance between the original point $ \bm{x} $ and the neighbor $ \bm{x}_{j} $, which in our study was used as the Euclidean distance.$ \sigma $ is a pre-defined kernel width; here we use $ 0.75\times \sqrt{M} $. Thus, the larger the distance, the smaller the weight of that neighbor in parameter estimation, and vice versa. The overview of this approach is shown in Algorithm~\ref{ace}.
\begin{algorithm}[!htb]
\caption{Anomaly Contribution Explainer (ACE)}
\label{ace}
\begin{algorithmic}[1]
\Require{black-box model \textit{f}, Number of neighbors \textit{N}}
\Require{The sample \textit{x} to be examined}
\Require{Distance kernel $ \pi_{\bm{x}} $, Number of feature \textit{K}} \Comment{$ \pi_{\bm{x}} $ measures the distance between a sample and $ \bm{x} $, which is used as the inverse weight}
\State{$ \mathcal{Z} \leftarrow \{\} $}
\For{$ j \in \{1, 2, 3, \dots, N\} $}
\If{$\bm{x} \in \mathbb{R}^{M}$}
\State{$ \bm{x}_{j} \leftarrow \textit{sample\_from}\, \mathcal{N}(\bm{x}, 0.01) $} \Comment{$ \mathcal{N} $ is a normal distribution}
\ElsIf{$\bm{x} \in \{-1,1\}^{M}$}
\State{$\bm{x}_{j} \leftarrow flip(\bm{x})\_with\,\operatorname{Bern} \left(\bm{0.1}\right)$} \Comment{$ \operatorname{Bern} $ is a Bernoulli distribution}
\EndIf
\State{$As_{j} \leftarrow \textit{f}(\bm{x}_{j})$}
\State{$ \mathcal{Z} \leftarrow \mathcal{Z}\cup \langle \bm{x}_{j}, \textit{f}(\bm{x}_{j}), \pi_{\bm{x}}(\bm{x}_{j}) \rangle $}
\EndFor
\State{$ \bm{w} \leftarrow \min\limits_{\bm{w}}\,loss(\mathcal{Z}) $} 
\State{Compute and sort $ w^{i}_{j}\cdot x^{i}_{j} $ for each $ i $} \Comment{$ i $ is the index for $ i $th feature}
\State{Pick the top $ K $ from the sorted results and calculate the contribution (Eq. \ref{cal_contribution})}
\end{algorithmic}
\end{algorithm}
\vspace{0em}

\subsection{Anomaly Contribution Explainer with KL Regularizer (ACE-KL)}

The ACE-KL model extends the ACE model by adding an additional regularizer. This regularizer tries to maximize the KL divergence between a uniform distribution and the calculated distribution of contributions of all the inspected features. By adding this regularizer to our loss function, our anomaly contribution explainer assigns contribution to inspected features in a more distinguishable way, inducing more contributions from the dominant ones and reducing the contributions from those less dominant ones. The KL divergence between a uniform distribution and a particular distribution takes the following form: 
\par\nobreak{\small
\begin{equation}
KL(P||Q)= \sum_{i}^{M}P(i)\log \dfrac{P(i)}{Q(i)},\quad P(i) \sim Uniform,
\end{equation}
}
where $ P(i) $ is the uniform distribution and $ Q(i) $ is the calculated distribution.

Hence, the loss function is formalized as following:
\begin{align*}
loss = &\dfrac{1}{N}\sum\limits_{j=1}^{N} \pi_{\bm{x}}(\bm{x}_{j})\cdot(\bm{w}^{\intercal}\bm{x}_{j}-As_{j})^{2}\\
&+ \alpha ||\bm{w}||^{2}_{2} - \beta KL_{j}(P||Q),\nonumber
\end{align*}
where $ \beta $ set to $50$ in this study is the hyper parameter to control the KL regularizer. This formulation forces the calculated distribution to be peaky. Therefore, in terms of contributions, those features that contribute most get better explained than others. Intuitively, this characteristic yields a better visualization for security analysts in real applications. 

Further, a merit that our ACE-KL model retains is that the new loss function is still a convex function. We sketch the proof by taking advantage of the \textit{Scalar Composition Theorem} \cite{Boyd2004}:
\begin{corollary}
The loss function of ACE-KL model is a convex function, w.r.t. its model parameters.
\end{corollary}
\begin{proof}
The formulation of the loss function for ACE-KL consists of two parts: a regular ridge regression and an additional regularizer. It is trivial to show a ridge regression is a convex function w.r.t. its parameters.\\
Now we show the additional regularizer is also a convex function (for each $j$):
\par\nobreak{\small
\begin{align*}
 KL(P||Q) & =   \sum\limits_{i}^{M}P(i)\log \dfrac{P(i)}{Q(i)}\\
& =   \sum\limits_{i}^{M}P(i)\log P(i) - P(i)\log Q(i)\\
& =  \sum\limits_{i}^{M} C - c\log Q(i) \\
&\textit{($ P(i) \sim Uniform $ , so $ P(i) $ is a constant)}\nonumber\\
& =  \sum\limits_{i}^{M} C - c\log \dfrac{\log(1+e^{w^{j}\cdot x^{j}})}{\sum\limits_{j=1}^{M}\log(1+e^{w^{j}\cdot x^{j}})}\\
& =  \sum\limits_{i}^{M} C - c\log\log(1+e^{w^{i}\cdot x^{i}})+c\log\sum\limits_{j=1}^{M}\log(1+e^{w^{j}\cdot x^{j}}),
\end{align*}
}
where $\log(1+e^{w^{i}\cdot x^{i}})$ is convex. By the \textit{Scalar Composition Theorem}, the $KL$ regularizer is convex. Then a linear combination of the convex ridge regression part and the $KL$ regularizer retain the property of convexity.
\end{proof}


%% file: exp.tex

\section{Experiments and Results}
\subsection{Data sets}

We validate our methods on three security related data sets. The first data set is the CERT Insider Threat v6.2 (abbreviated as CERT) \cite{lindauer2014,glasser2013}. It is a synthetically generated, realistic data set consisting of application-level system logs, such as HTTP get/post requests, emails and user login/logout events. 
The second data set contains \textit{pcap} traces from UNB \cite{Shiravi2012}, which we converted to netflow logs using \textit{nfdump}. It is partially labelled with port scanning and intrusion events.  
Lastly, the third data set--AndroidMalware \cite{Zhou2012}--is a collection of about 1,200 malwares observed on Android devices. 

\subsection{Feature Extraction}

To evaluate our methods, we build anomaly detection models on these data sets.  
Note that the models can be supervised or unsupervised as long as they produce an anomaly score. Furthermore, while we ensure these models have reasonable accuracy, building the best possible anomaly detection models for these data sets is not the focus of this work.
We extract the following features from the data sets.

\noindent \textbf{CERT.} Similar to a previous study \cite{Tuor2017}, we extract count features conditioned on time of day, where a day is uniformly discretized into four intervals. In our experiments, we use one day as the smallest time window and each example is the composite record of day-user. We examine three different Internet activities: \textit{``WWW visit"}, \textit{``WWW upload"} and \textit{``WWW download"}. Hence, in this setting, the total number of features are $ 3\times4= 12$, and so is the input dimensionality of the autoencoder model (one of the baselines). 

\noindent \textbf{UNB Netflow:} We extracted 108 features that can be categorized into three sets: Count, Bitmap, and Top-K. The Count features count the number of bytes/packets for incoming and outgoing traffic; the Bitmap features include type of services, TCP flags, and protocols; the Top-K features encode the IP addresses with traffic flows ranked in top k over all the addresses. 

\noindent \textbf{AndroidMalware:} 122 binary features are extracted, mainly related to frequent permission requests from apps. 

\subsection{Evaluation Metrics}
We consider the contributions to be a distribution over features. To quantitatively evaluate contributions produced by a method, we use its Kullback-Leibler (KL) divergence with respect to ground truth contributions. The KL divergence measures how one probability distribution diverges from another probability distribution.
Given the distribution of modeled contributions, $ Q, $
and the ground truth contributions distribution of the data point, $ P, $ the KL divergence is formulated as:
\par\nobreak{\small
\begin{equation}
KL(P||Q)= \sum_{i}^{M}P(f_i)\log \dfrac{P(f_i)}{Q(f_i)},
\end{equation}
}
where $ f_i $ is the $ i $th feature. The lower the KL divergence, the closer the modeled contribution is to the real contribution for that data point. Note that this KL divergence metrics is different from the regularizer term in ACE-KL, which forces the formulated distribution away from a uniform distribution. 

\subsection{Baseline Methods}

We use LIME \cite{Ribeiro2016a} as our main baseline. We consider it representative of similar methods since it is recent and well cited.
LIME only works for classification problems; however, most anomaly detection problems require an anomaly score to express the confidence of detection. We therefore extend LIME to support regression problems. This extension is straightforward: in classification problems, each feature is mapped onto a classification class by looking at the estimated weights of each feature to decide the importance of that feature to the particular class. In a regression problem, we can assume it to be a one-class classification problem. We therefore transform LIME from multi-class classification to a one-class problem, and examine the importance of each feature to the anomaly score. 




\subsection{Evaluation on CERT}
To evaluate ACE and ACE-KL on CERT, first we train an autoencoder  as our black-box model, although in principle it could be any model. Its anomaly score ($As$) on a data point is computed as the \textit{mean squared error} (MSE) between the input and the output vector. 
In addition to applying ACE and ACE-KL, we compute feature contributions from the autoencoder model using the reconstruction error of each of the inputs, similar to \cite{Tuor2017}. Thus, the autoencoder model serves as an additional baseline.
While the CERT data set has some anomalies, we also artificially inject some by perturbing the input features. 
The data set contains two years of activities. We use the first year of the data set for injected anomalies detection, as it has no anomaly marked. 
We also detect anomalies present in the second year.

\subsubsection{Evaluation on Injected Anomalies}
We perturb individual features and groups of features. 
However, due to space limitations, we only present perturbation of groups of five features. The rest of the results on injected anomalies are described in the appendix. 
\paragraph{Multiple Feature Perturbation} 
The synthetic anomalies are created as follows.
We first calculated the mean values of each feature based on the non-preprocessed raw data, and draw from a Poisson distribution based on the mean value of each feature using
$
P(x) = e^{-\lambda}\frac{\lambda^{k}}{k!}.
$
This sampling approach ensures that first, all the synthesized features are integers; second, the original value is around the mean of the raw data. After we sample from this distribution, we perturbed it by adding a value $ \lambda$ to the feature $x$: $ x' = x + \lambda $. This new value's expectation is $ \mathbb{E}[x']=2\lambda, $ which exceeds the mean value of $ \lambda $ by a large magnitude, thus this perturbation can represent an anomaly from the original data.

We randomly chose five features to perturb. 
Each feature of a data point is standardized to $\mathcal{N}(0, 1)$ using the mean and the standard deviation of that feature of the training set, and fed into the trained black-box to create an anomaly score. 
 
 The results are shown in Figure ~\ref{fig:fivevpert}. 
 ACE accurately identified the contributions in both anomalies, and performed significantly better than both baselines considered according to the KL-divergence metric. ACE-KL, while not as accurate as ACE, highlights the top contributors. 
 
\begin{figure}[!htb]
    \centering
    \begin{subfigure}[b]{0.48\textwidth}
        \includegraphics[width=\textwidth]{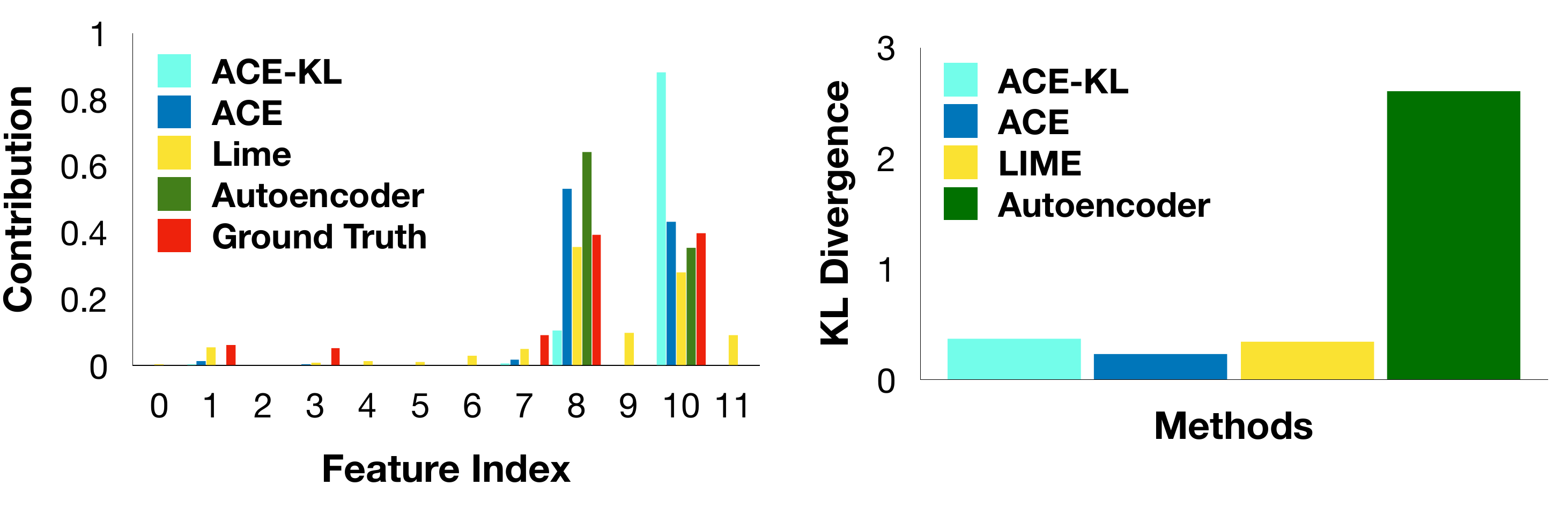}
        \caption{Features 1, 3, 10, 8, 7 were perturbed.}
        \label{fig:eg0}
    \end{subfigure}
    \begin{subfigure}[b]{0.48\textwidth}
        \includegraphics[width=\textwidth]{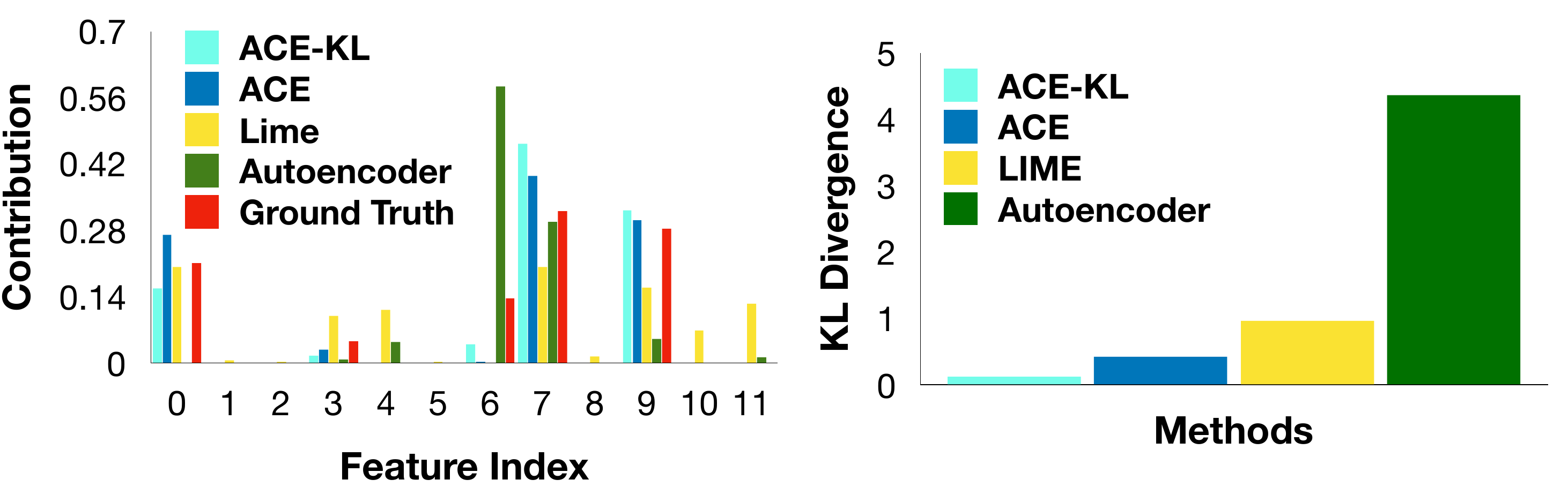}
        \caption{Features 6, 0, 9, 7, 3 were perturbed.}
        \label{fig:eg1}
    \end{subfigure}
    \caption{\textbf{(left)} Feature contributions on two synthetic examples, with perturbation on five randomly chosen features. Contribution is the percentage of a feature towards the anomaly score. \textbf{(right)} KL-divergence of each method with respect to the ground truth.}
\label{fig:fivevpert}
\end{figure} 
\subsubsection{Evaluation on Real Anomalies}
The CERT data set contains labeled
 scenarios where insiders behave maliciously. 
 Figure~\ref{fig:realinsider} shows contribution analysis on the days that have the malicious activities. In Figure \ref{fig:realinsider}(a) and \ref{fig:realinsider}(c), feature $7$ captures the malicious activities, while in Figure \ref{fig:realinsider}(b) feature $8$ is the ground-truth anomalous feature.
The experimental results and the corresponding KL-divergence are shown in Figure~\ref{fig:realinsider}. ACE and ACE-KL accurately capture the feature responsible for the anomalies. Both ACE and ACE-KL have significantly lower KL divergence, outperforming the baselines. 
\begin{figure}[!htb]
    \centering
    \begin{subfigure}[b]{0.48\textwidth}
        \includegraphics[width=\textwidth]{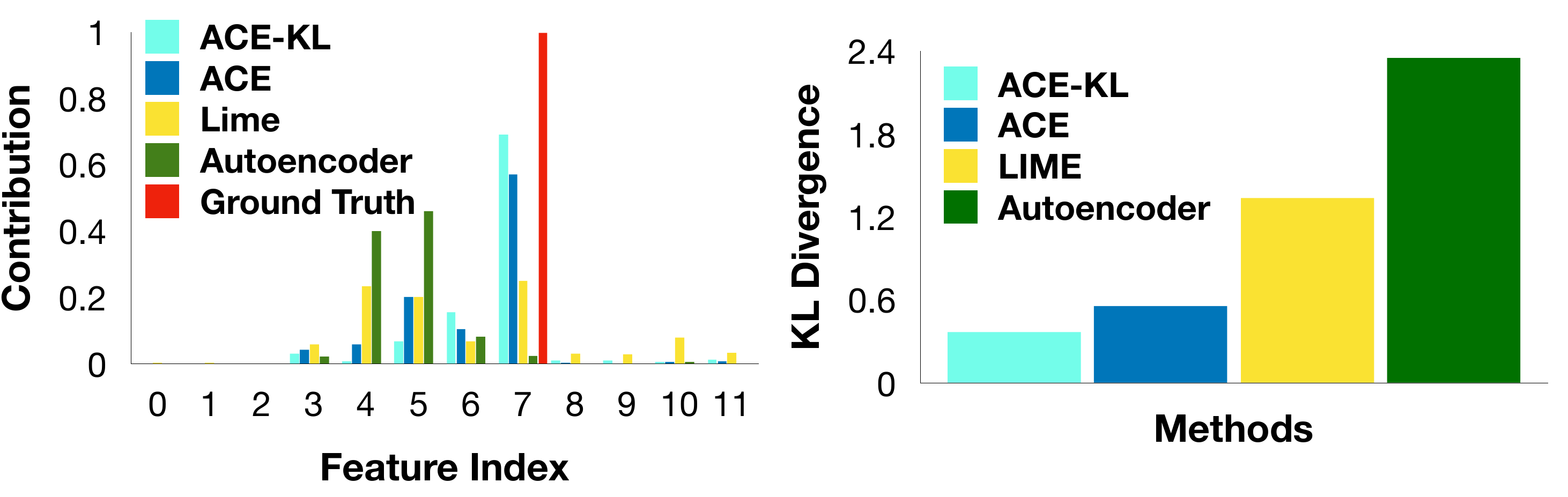}
        \label{fig:ceg0}
        \caption{\footnotesize WWW download anomaly, feature 7, day 398.}
    \end{subfigure}
    \begin{subfigure}[b]{0.48\textwidth}
        \includegraphics[width=\textwidth]{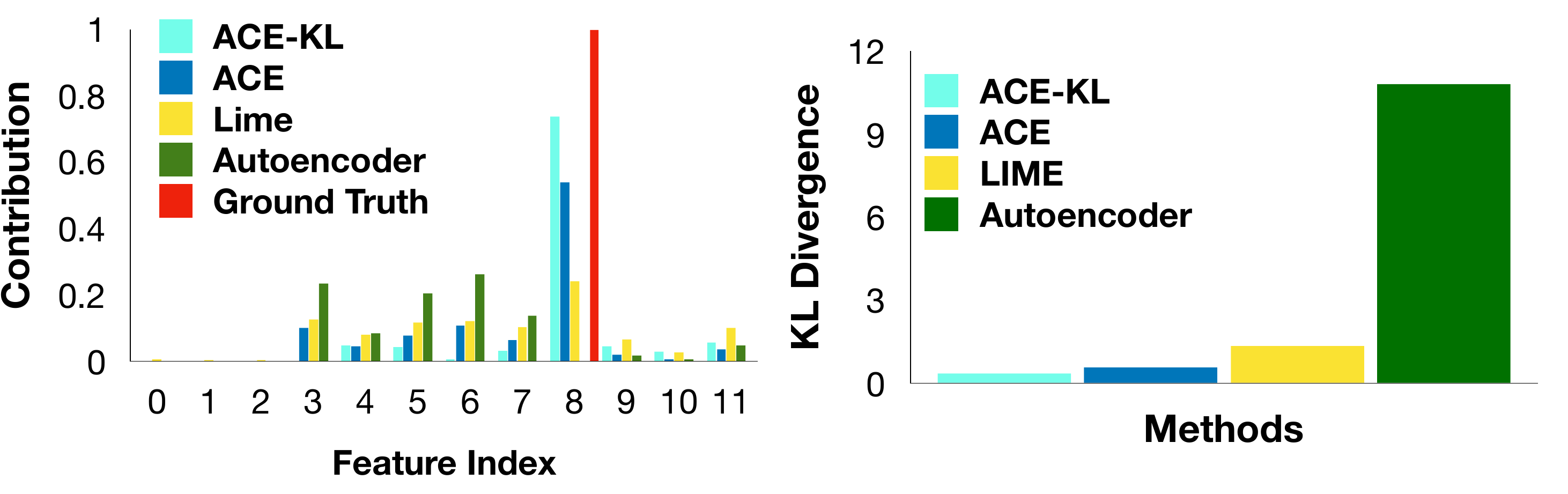}
        \label{fig:ceg1}
        \caption{\footnotesize WWW upload anomaly, feature 8, day 404.}
    \end{subfigure}
    \begin{subfigure}[b]{0.48\textwidth}
        \includegraphics[width=\textwidth]{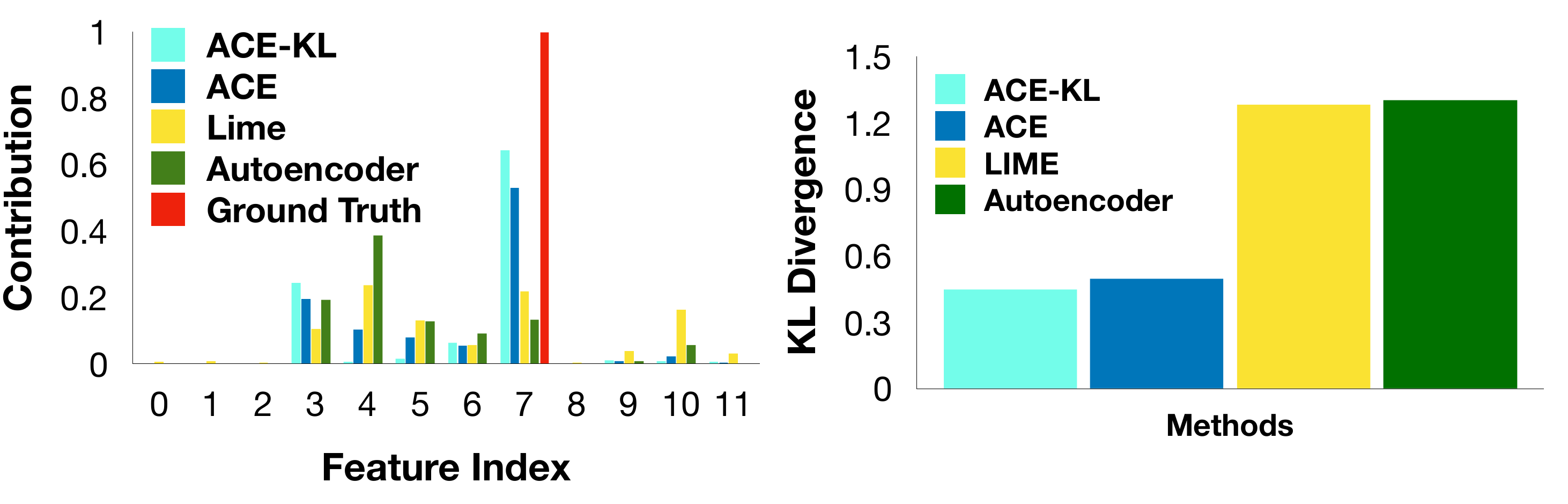}
        \label{fig:ceg2}
        \caption{\footnotesize WWW download anomaly, feature 7, day 409.}
    \end{subfigure}
    \caption{Three real anomalies in the CERT data set. \textbf{(left)} Feature contributions using ACE, ACE-KL and two baselines. \textbf{(right)} KL-divergence between feature contributions computed by the methods and the ground truth contributions. ACE and ACE-KL has the most similar contribution as the ground truth (which is always 1.0).
}
\label{fig:realinsider}
\end{figure} 

\subsection{Evaluation on UNB Netflow Data Set}

This section presents the evaluation of ACE and ACE-KL on UNB Netflow with similar settings as CERT. A separately trained autoencoder is used as a black-box anomaly detection model. Due to space limitations, we present results of applying ACE and ACE-KL to only two anomalies here. 
Table~\ref{tab:x} provides a short description of the top 10 features that are useful to interpret the results. Figure \ref{fig:netcontribution} shows the feature contributions for the anomalies, and Table~\ref{tab:table1} provides details on the feature values and their contributions. Since the annotation is at the packet level, it is not easy for a person to manually determine the root cause for the anomaly.


\begin{table*}[t]
\small
\begin{subtable}{1.0\textwidth}
\centering
\begin{tabular}{ |p{0.3\textwidth}|p{0.5\textwidth}| }
  \hline
	\# std src ports & Number of standard source ports\\
	avg std src ports per dst ip & Average number of standard source ports per destination IP\\
	protos out 3 & Third bit in the Protocol feature (3 bit feature indicating TCP, UDP, or Other)\\
	top1 out & Top 1st outgoing IP address (in terms of bytes)\\
	top3 out & Top 3rd outgoing IP address (in terms of bytes)\\
  \hline
\end{tabular}
\caption{\small Features for outgoing flows (when IP is source)}
\label{tab:xa}
\end{subtable}%
\vfill
\begin{subtable}{1.0\textwidth}
\centering
\begin{tabular}{ |p{0.3\textwidth}|p{0.5\textwidth}| }
  \hline
 	max duration in & Maximum incoming flow duration\\
	\# std dst ports & Number of standard destination ports\\
	avg std dst ports per src ip & Average number of standard destination ports per source IP\\
	Flags in 3 & Third bit in the flags field\\
    total duration in & Total duration of the incoming flows\\
  \hline
\end{tabular}
\caption{\small Features for incoming flows (when IP is destination)}
\label{tab:xb}
\end{subtable}%
\caption{Short descriptions of top features in the results.}
\label{tab:x}
\end{table*}

For anomaly 1, the highest contributing feature is `\textbf{max\_duration\_in}', which is the maximum duration of an incoming flow into this IP address (192.168.1.103). After examining the netflow records, we found that the high value for this feature was related to long-lived (i.e., persistent) TCP connections. Although benign, this was an unusual activity relative to other recorded traffic. The other high values correspond to the number of standard source and destination ports. This was found to be related to a port scanning activity, which was not previously discovered, i.e., was not labeled. Anomaly 2 is almost exactly similar to Anomaly 1 with a similar port scanning activity. 


\begin{figure}[!htb]
    \centering
    \begin{subfigure}[b]{0.3\textwidth}
        \includegraphics[width=\textwidth]{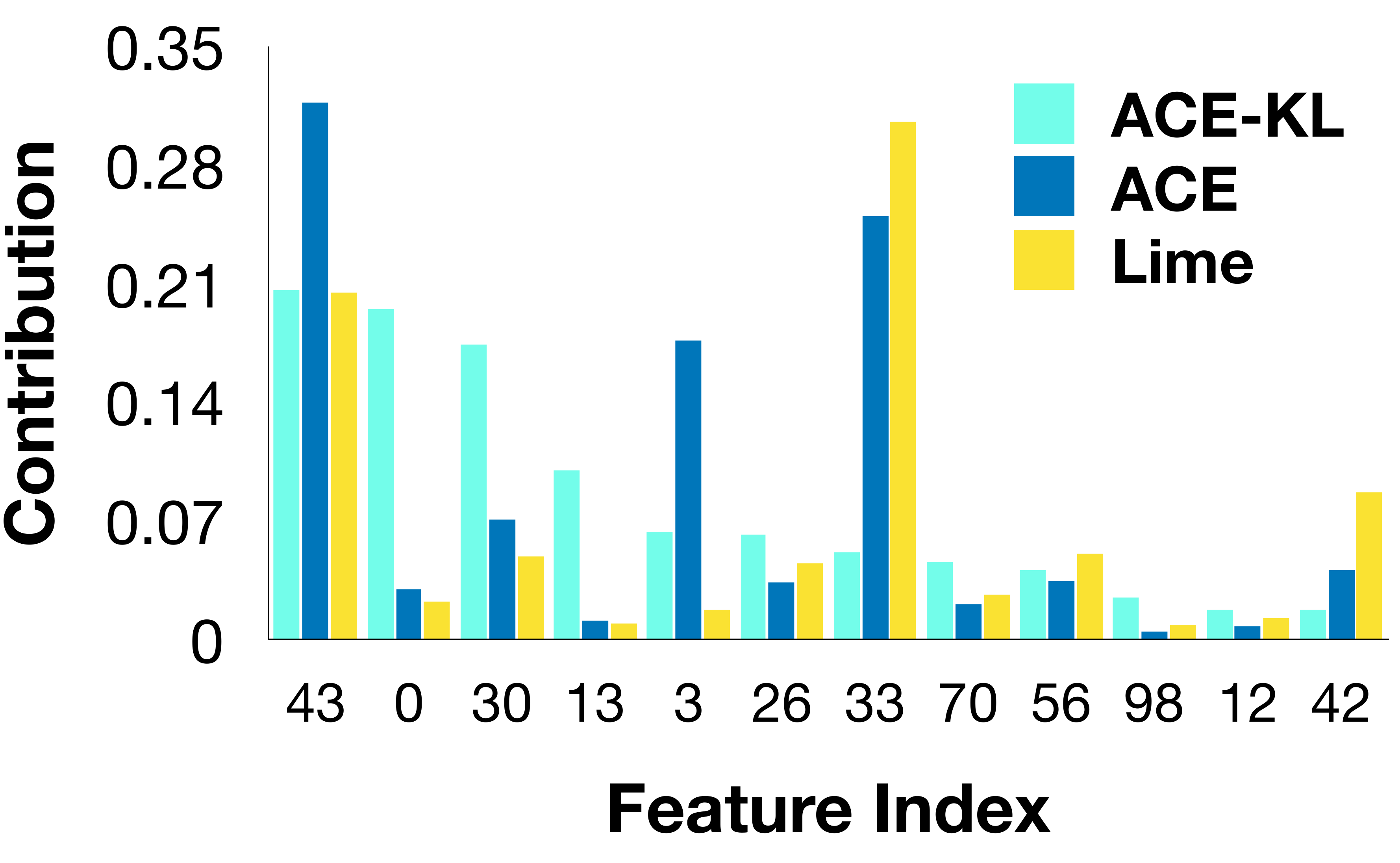}
        \caption{Anomaly 1}
        \label{fig:net1}
    \end{subfigure}
    \begin{subfigure}[b]{0.3\textwidth}
        \includegraphics[width=\textwidth]{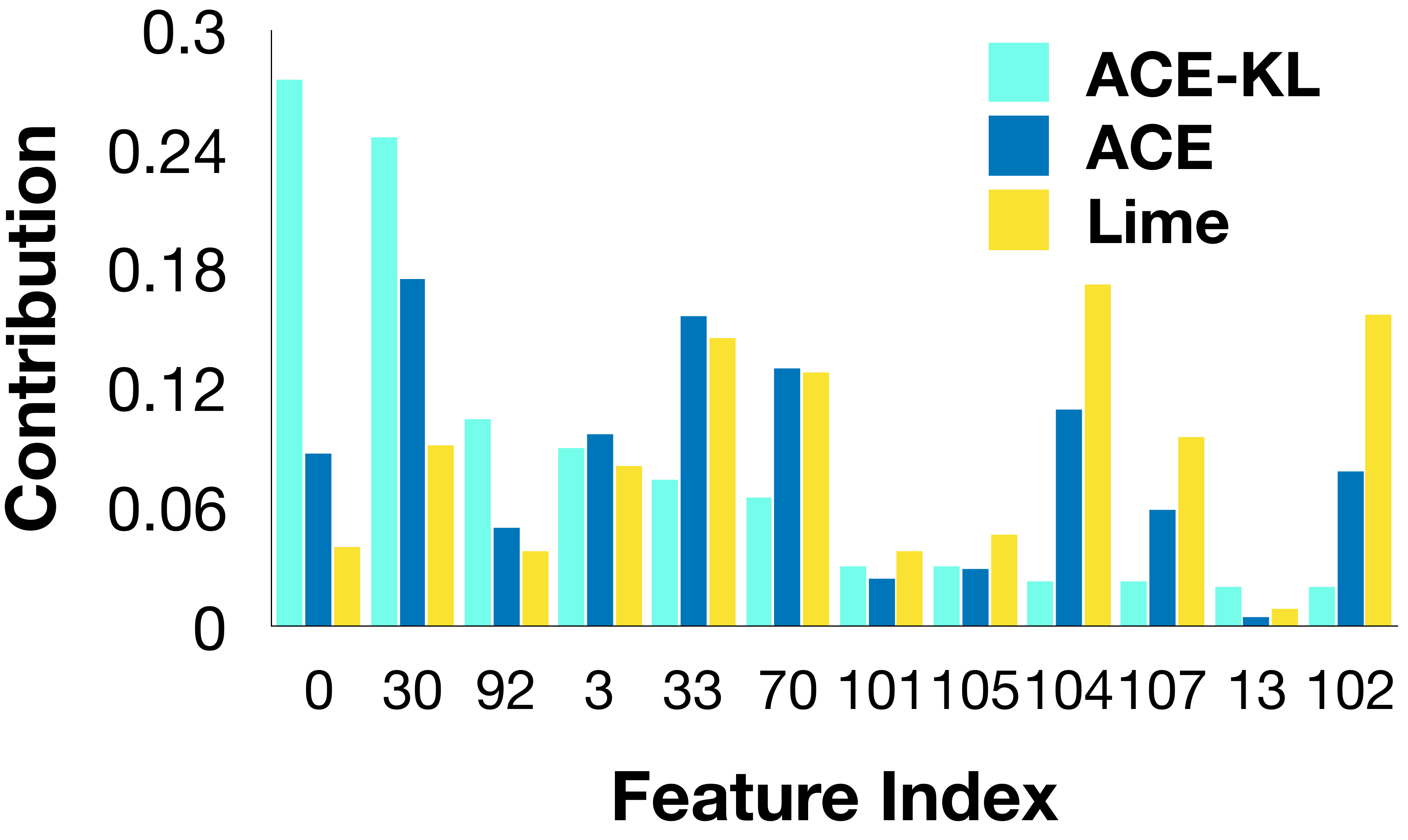}
        \caption{Anomaly 2}
        \label{fig:net2}
    \end{subfigure}
\caption{Contribution analysis on two anomalies in netflow data.}
\label{fig:netcontribution}
\end{figure}
\begin{table*}[t]
\small
\begin{subfigure}[b]{0.5\textwidth}
\centering
\begin{tabular}{ |l|p{0.3\textwidth}|c|c|c| }
  \hline
  Index & Feature Name & ACE-KL & ACE & value\\
  \hline
  43 & max duration in & \textbf{0.207} & \textbf{0.317} & \textbf{239.961}\\
  0 & \# std src ports & \textbf{0.195} & 0.030 & \textbf{158}\\
  30 & \# std dst ports & \textbf{0.174} & 0.071 & \textbf{156}\\
  13 & max duration out & 0.100 & 0.011 & 240.085\\
  3 & avg std src ports per dst ip & 0.064 & \textbf{0.177} & \textbf{1}\\
  26 & min n bytes out & 0.062 & 0.034 & 20\\
  33 & avg std dst ports per src ip & 0.052 & \textbf{0.250} & \textbf{1}\\
  70 & protos out 3 & 0.046 & 0.021 & 1\\
  56 & min n bytes in & 0.041 & 0.035 & 20\\
  98 & top1out & 0.025 & 0.005 & \tiny{192.168.1.101}\\
  12 & total duration out & 0.018 & 0.008 & 62154.867\\
  42 & total duration in & 0.018 & 0.041 & 44405.557\\
  \hline
\end{tabular}
\caption{\footnotesize Anomaly1: 192.168.1.103, Sunday}
\label{tab:table1_b}
\end{subfigure}
\hspace{0.5em}
\begin{subfigure}[b]{0.5\textwidth}
\centering
\begin{tabular}{ |l|p{0.3\textwidth}|c|c|c| }
  \hline
  Index & Feature Name & ACE-KL & ACE & value\\
  \hline
  0 & \# std src ports & \textbf{0.275} & 0.087 & \textbf{158}\\
  30 & \# std dst ports & \textbf{0.246} & \textbf{0.175} & \textbf{156}\\
  92 & flags in 3 & \textbf{0.104} & 0.0496 & \textbf{0}\\
  3 & avg std src ports per dst ip & 0.090 & 0.097 & 0\\
  33 & avg std dst ports per src ip & 0.074 & \textbf{0.156} & \textbf{0}\\
  70 & protos out 3 & 0.065 & \textbf{0.130} & \textbf{1}\\
  101 & top4 out & 0.030 & 0.024 & \tiny{67.220.214.50}\\
  105 & top3 in & 0.030 & 0.029 & \tiny{61.112.44.178}\\
  104 & top2 in & 0.023 & 0.109 & \tiny{125.6.176.113}\\
  107 & top5 in & 0.023 & 0.059 & \tiny{192.168.5.122}\\
  13 & max duration out & 0.020 & 0.005 & 280.53\\
  102 & top5 out & 0.020 & 0.078 & \tiny{203.73.24.75}\\
  \hline
\end{tabular}
\caption{\footnotesize Anomaly2: 192.168.2.110, Sunday}
\label{tab:table1_c}
\end{subfigure}
\caption{Contributions and feature values for top two anomalies in netflow data. The contributions in bold are the top ones.}
\label{tab:table1}
\end{table*}

Identifying anomalies from netflow records is a time consuming and laborious (and thus error-prone) task. Since our method is able to systematically provide a basic explanation (in terms of features) of why some of the anomalies were identified as such, the internal security expert who we consult is convinced that our method is trustworthy and practical. As noted earlier, several of the IP addresses exhibited multiple distinct anomalous behaviors, as well as benign characteristics such as the persistent TCP connections for certain applications. As future work the expert recommended investigating how to systematically discern between multiple anomalies involving a single IP address, to make it easier for a security analyst to understand which are malicious and require their attention, and which are benign and can be ignored. This would accelerate an analyst's ability to respond faster to malicious activities, and therefore improve the security of the analyst's organizations.

\subsection{Evaluation on Android Malware Data Set}
Finally, we evaluate ACE and ACE-KL on the Android malware data set\cite{Zhou2012}. This data set captures various features related to app activities, including their installation methods, activation mechanisms as well as their susceptibility to carry malicious payloads. In this data set, each example is a numeric, binary vector of 122 dimensions, representing features for  malware detection. Peng et al. \cite{Peng2012} successfully built probabilistic generative models for ranking risks of those Android malwares in a semi-supervised learning setting by using a large amount of additional unlabeled data. The risk scoring procedure is a form of anomaly detection, and the risk scores equate to anomaly scores. Thus, in this evaluation, we used the pre-built hierarchical mixture of naive Bayes (HMNB) model \cite{Peng2012} as the black-box model to generate an anomaly score, and applied our approach to explain the anomaly. As the HMNB model calculates the likelihood of a malware in the population, we use the negative log-likelihood as the anomaly score.

We inspected the four malwares that obtained the highest anomaly score by using the pre-trained HMNB model. Before we analyzed the anomalies using ACE and ACE-KL, all the $0$s in the features were replaced by $-1$s, as a $0$ feature will result in a constant contribution of a feature.  The contributions of each feature is calculated using ACE and ACE-KL. The final results are presented in Fig~\ref{fig:malwarecontribution}. The feature indices are sorted by the contributions calculated by ACE, and we only show the top 10 features. In all four cases, ACE and ACE-KL produce consistent contributions, although their results differ from LIME. 

\begin{figure}[!htb]
    \centering
    \begin{subfigure}[b]{0.24\textwidth}
        \includegraphics[width=\textwidth]{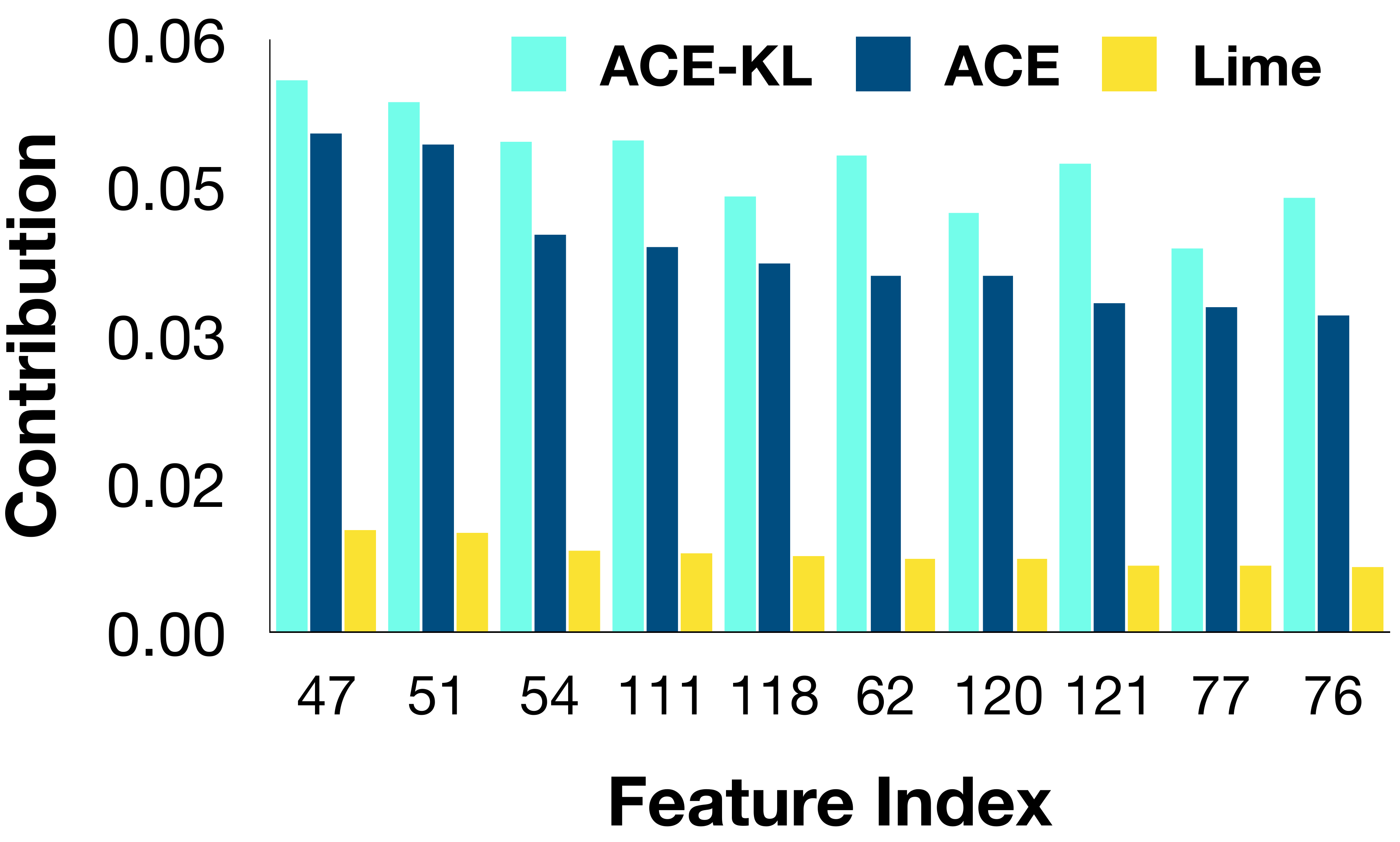}
        \caption{Anomaly 1}
        \label{fig:mw1}
    \end{subfigure}
    \begin{subfigure}[b]{0.24\textwidth}
        \includegraphics[width=\textwidth]{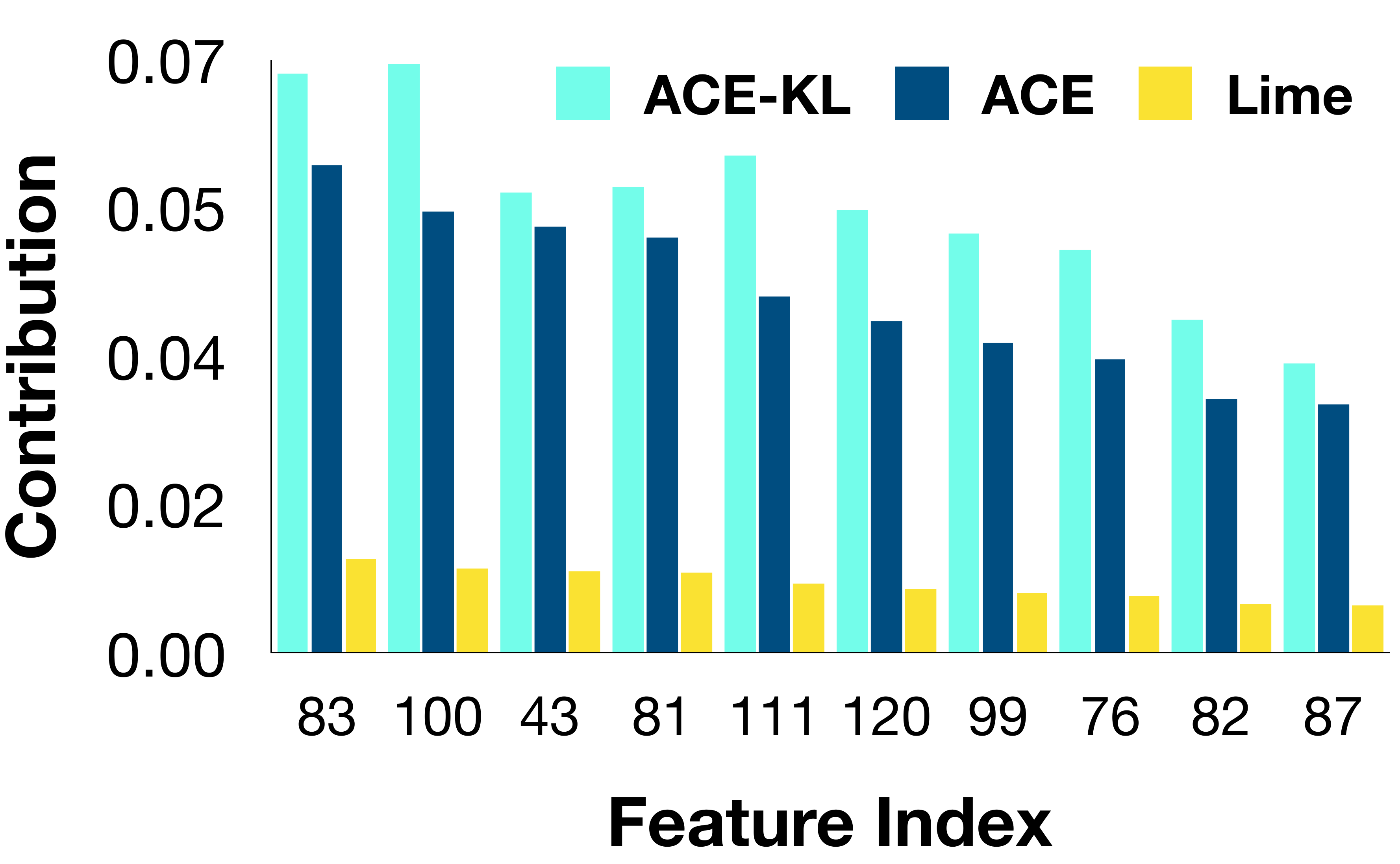}
        \caption{Anomaly 2}
        \label{fig:mw2}
    \end{subfigure}
    \begin{subfigure}[b]{0.24\textwidth}
      \includegraphics[width=\textwidth]{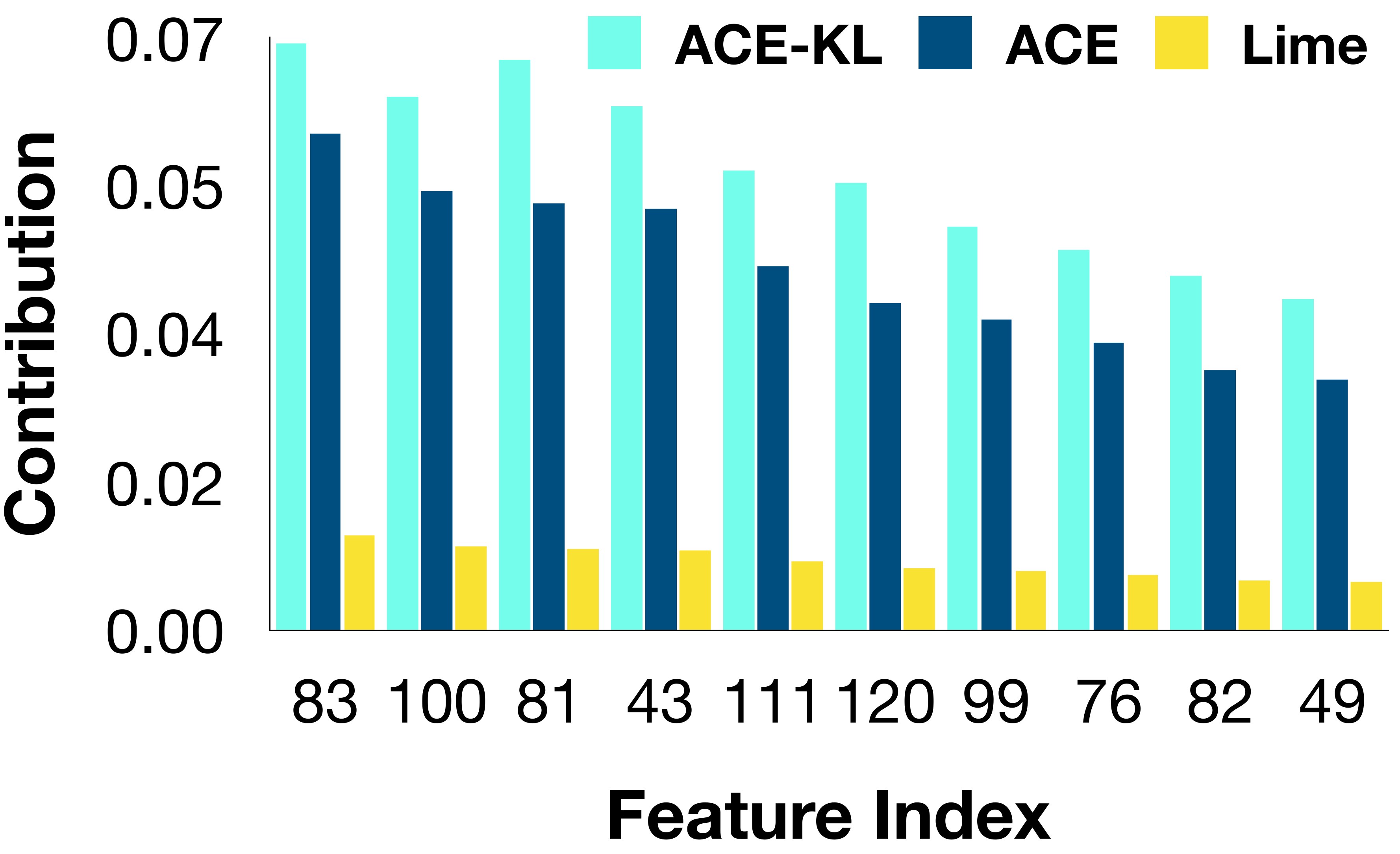}
      \caption{Anomaly 3}
      \label{fig:mw3}
    \end{subfigure}
    \begin{subfigure}[b]{0.24\textwidth}
        \includegraphics[width=\textwidth]{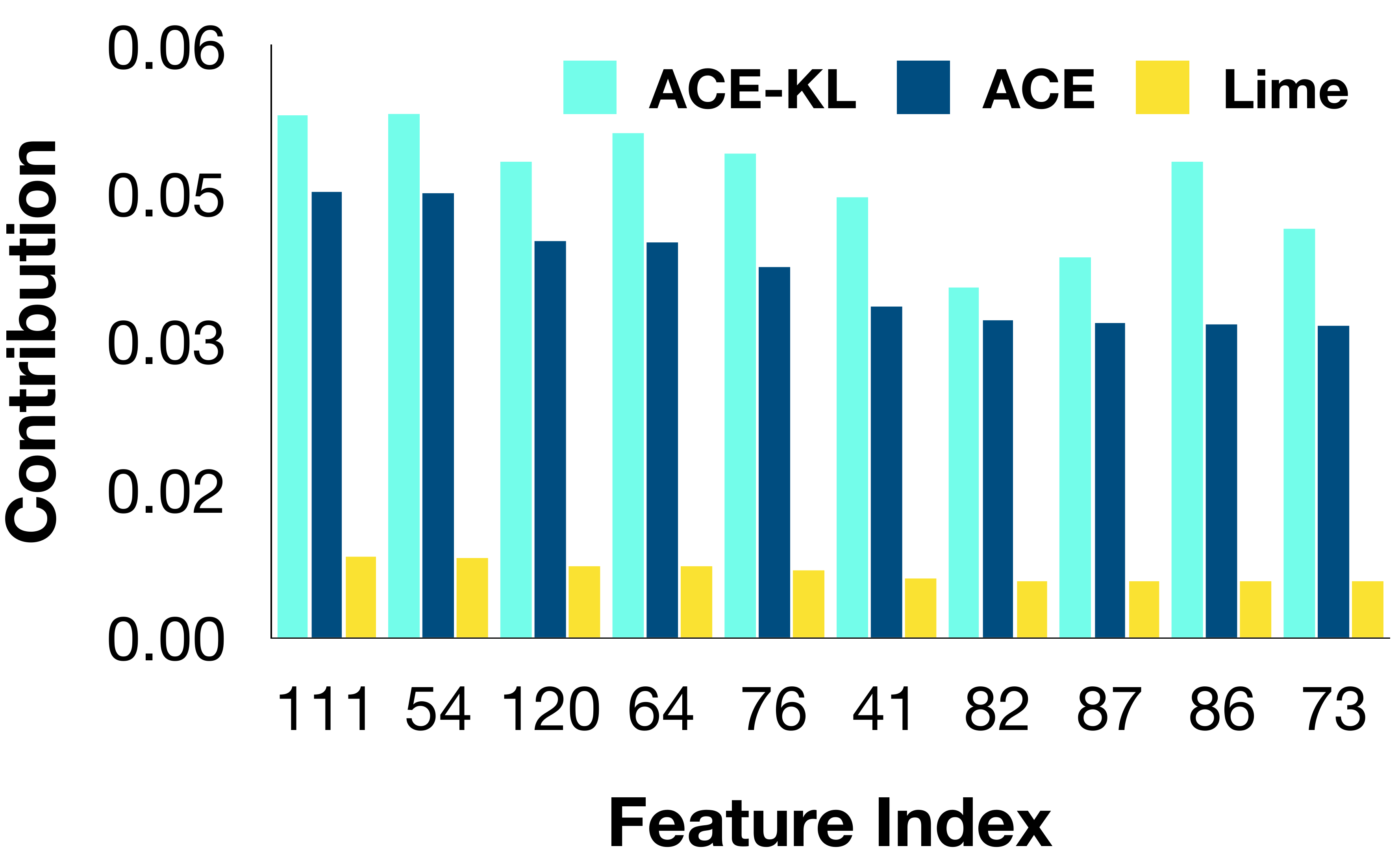}
        \caption{Anomaly 4}
        \label{fig:mw4}
    \end{subfigure}
\caption{Contribution analysis on four anomalies in Android malware data. We only show the top 10 features that contribute most significantly to the anomaly score in terms of percentage.}
\label{fig:malwarecontribution}
\end{figure}

To gain a better understanding of the difference between ACE, ACE-KL and LIME, we show the probability mass graph of all the features as the contributions for Malware 1 in Figure \ref{fig:androiddist}. As stated, both ACE and ACE-KL identified the same features that contribute most to the anomaly. Further, the contribution distribution induced by ACE-KL forms a more skewed distribution, highlighting those features that contribute most to the anomaly while neglecting those with small contributions. In contrast, the contribution distribution calculated by LIME is relatively flat compared to ACE and ACE-KL.

\begin{figure}[!htb]
    \centering
    \includegraphics[width=0.5\textwidth]{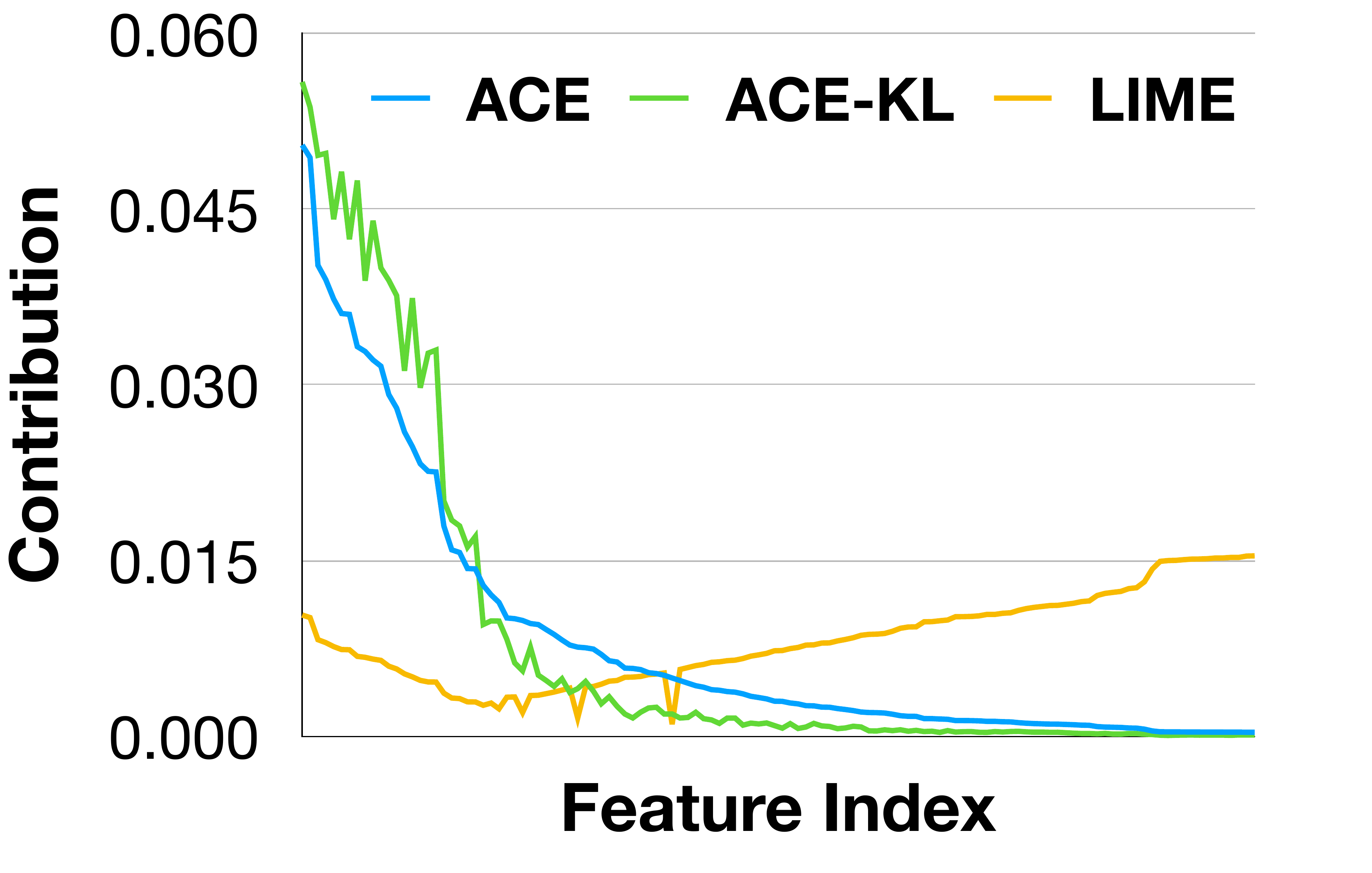}
    \caption{Probability mass function for each feature in Malware 1. This forms the whole contribution distribution to the anomaly score for this Malware.}
\label{fig:androiddist}
\end{figure}

\textbf{Anomaly Remediation:}
Although the Android Malware data set is labeled with anomalies, the contributing features to these anomalies are unknown, making it difficult to validate our results. 
To get some degree of validation, we conducted additional experiments which we call ``anomaly remediation". Essentially, we change input feature values (flip binary features) to repair a particular anomaly, i.e., to see if the anomaly score reduces significantly for a particular example. 

In these experiments, we first flip the top 10 binary contributing features detected by ACE (or ACE-KL, in all four cases the top 10 features are identical for ACE and ACE-KL) for the four anomalies, and the top 10 features selected by LIME. We also randomly sample 10 features among all the 112 features, and flip them. Our conjecture is as follows: if the true features causing the Android app to be classified as malware correspond to those detected by ACE, then fixing the anomaly (by flipping the features) should result in much higher drop in the anomaly score than if the 10 features were randomly picked. 
The results of our experiments are summarized in Figure \ref{fig:androidamendment}.

\begin{figure}[!htb]
    \centering
    \includegraphics[width=0.5\textwidth]{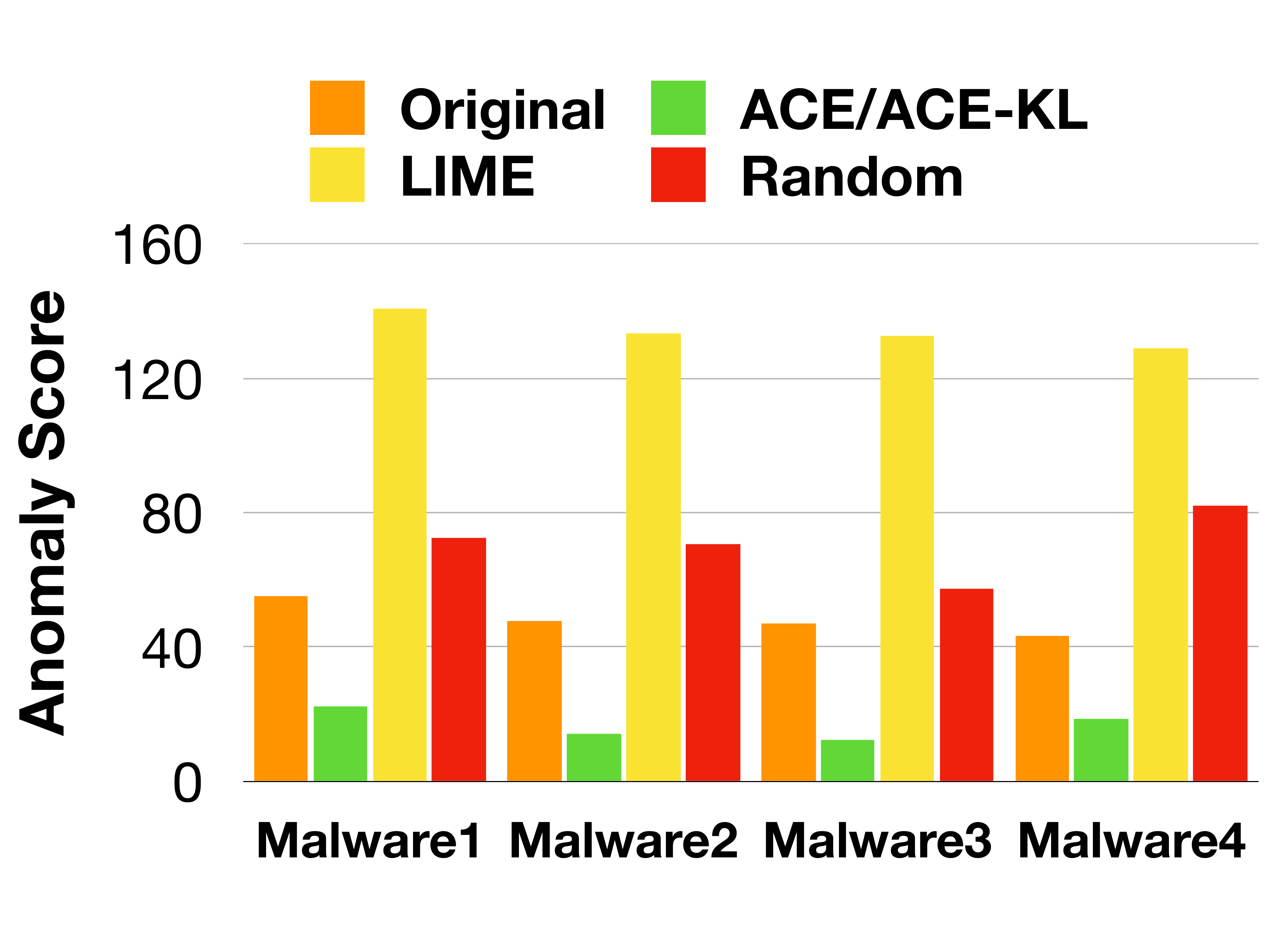}
    \caption{Comparison of original anomaly scores, the scores after anomaly remediation for ACE/ACE-KL and LIME, and the scores after random feature selection. Remediation with ACE/ACE-KL greatly reduces the anomaly score after the correct features contributing mostly to the anomaly score are identified, while LIME and randomly choosing a feature to remedy increase the anomaly score (by flipping a feature that does not contributing significantly to the anomaly originally).}
\label{fig:androidamendment}
\end{figure}
As can be seen in Figure \ref{fig:androidamendment}, by flipping the top 10 features detected by ACE/ACE-KL, the anomaly scores generated by the well-trained black-box model significantly drop for all four malwares. If we randomly pick the 10 features, the anomaly scores increase for all the four malwares. This is expected since only a small number of features are likely to cause a particular anomaly, and random sampling is more likely to select non-contributing features. Surprisingly, remediation of the top 10 features selected by LIME result in a higher increase in the score than random selection, which further shows LIME is not suitable for this problem.
We suspect this is likely because LIME only considers the weight vector of the regression framework, neglecting the importance of whether the feature is 1 or -1. 

%% file: concl.tex
\section{Conclusions}
In this paper we proposed methods for explaining results of complex security anomaly detection models in terms of feature contributions, which we define as the percentage of a particular feature contributing to the anomaly score. Based on our experimental results on synthetic and real data sets, we demonstrated that ACE consistently outperforms the baseline approaches for anomaly detection explanation. ACE-KL helps provide a simpler explanation focusing on the most significant contributors. Both approaches have valuable applications in the area of anomaly detection explanation in security. In the future, we plan to further validate our approach in other security problems and other domains.

%% file: appendix.tex
\section{Appendix: Additional Experimental Results}
\subsection{Perturb One Feature for CERT}
In this section we present the experimental results of applying different anomaly explanation methods by perturbing only one of the twelve features in the CERT data set as an injected anomaly in Figure \ref{fig:onepert}. The left column shows the contributions of each feature calculated by the anomaly explanation method, and the right column the KL-divergence between the distribution of the calculated contributions and the uniform distribution. As we can see from Figure \ref{fig:onepert}, ACE and ACE-KL perform well across all six examples consistently, while Autoencoder and LIME fail to capture the contribution of the anomaly in some cases even there is only one anomaly feature.

\begin{figure}[!ht]
    \centering
    \begin{subfigure}[b]{0.48\textwidth}
        \includegraphics[width=\textwidth]{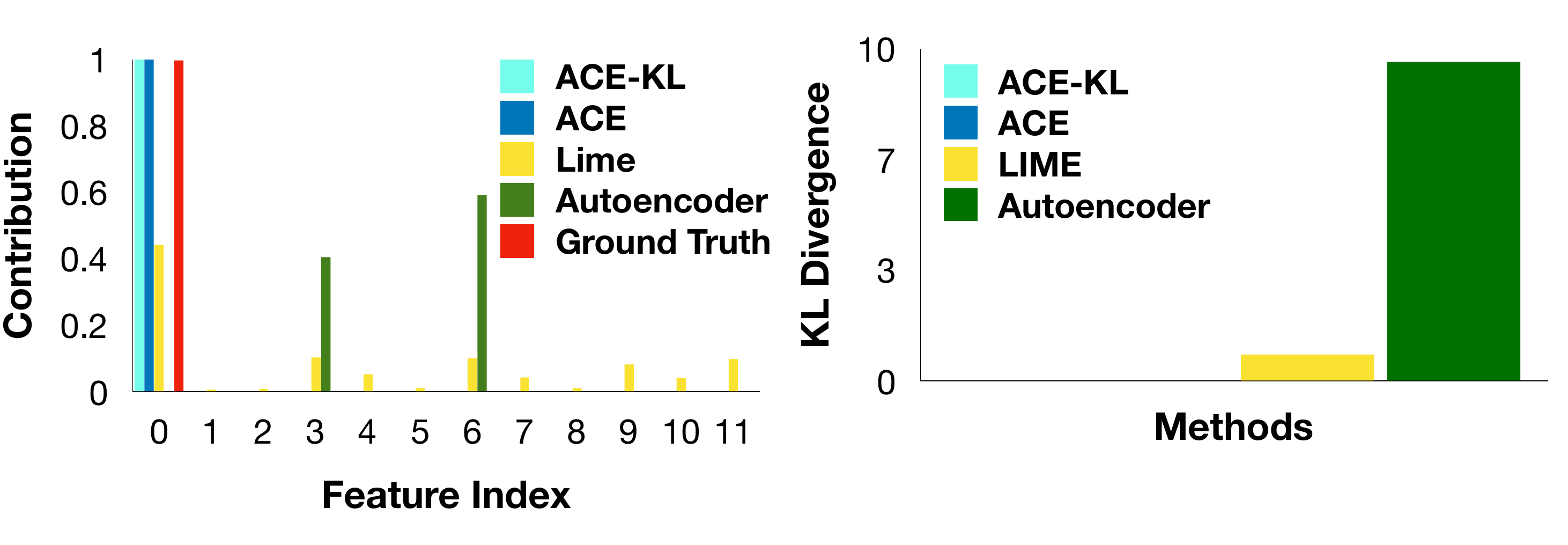}
        \caption{Features 0 is perturbed.}
    \end{subfigure}
    \begin{subfigure}[b]{0.48\textwidth}
        \includegraphics[width=\textwidth]{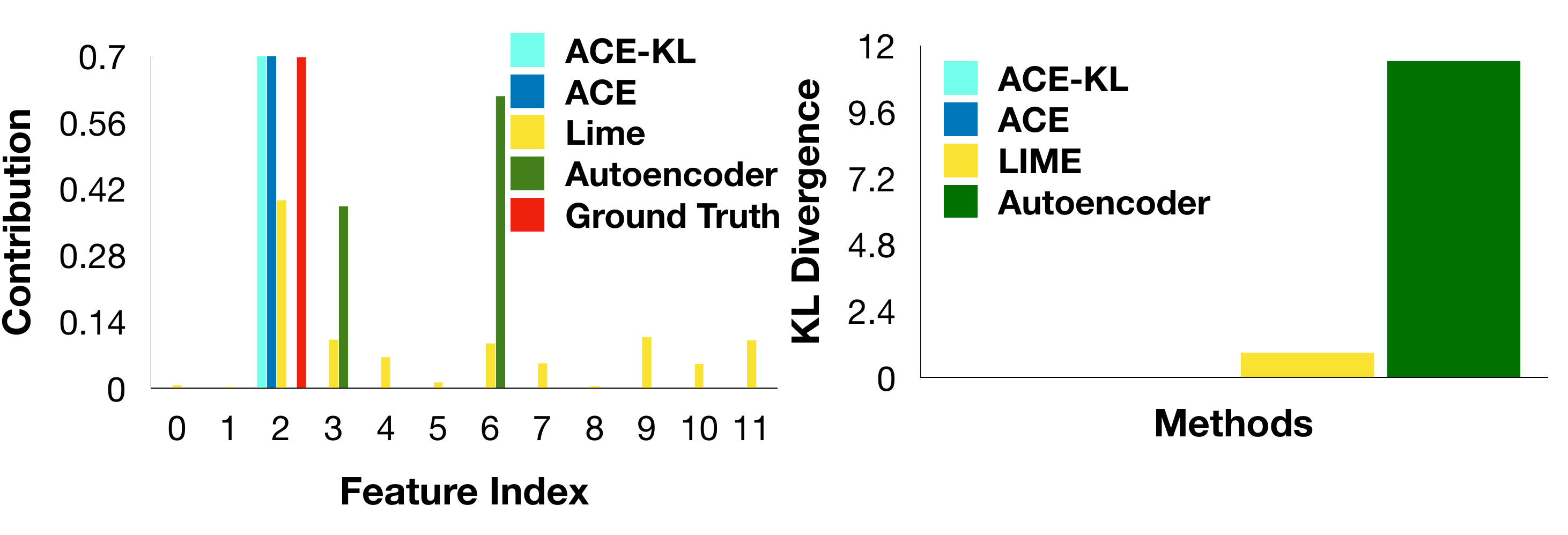}
        \caption{Features 2 is perturbed.}
    \end{subfigure}
    \begin{subfigure}[b]{0.48\textwidth}
        \includegraphics[width=\textwidth]{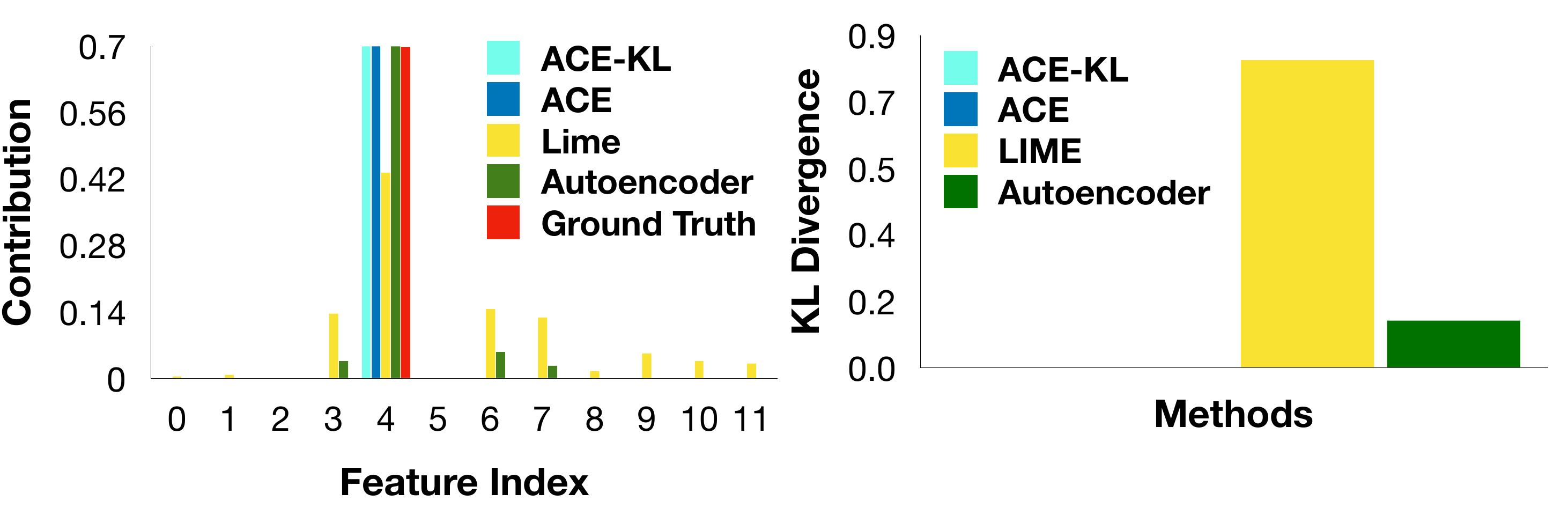}
        \caption{Features 4 is perturbed.}
    \end{subfigure}
    \begin{subfigure}[b]{0.48\textwidth}
        \includegraphics[width=\textwidth]{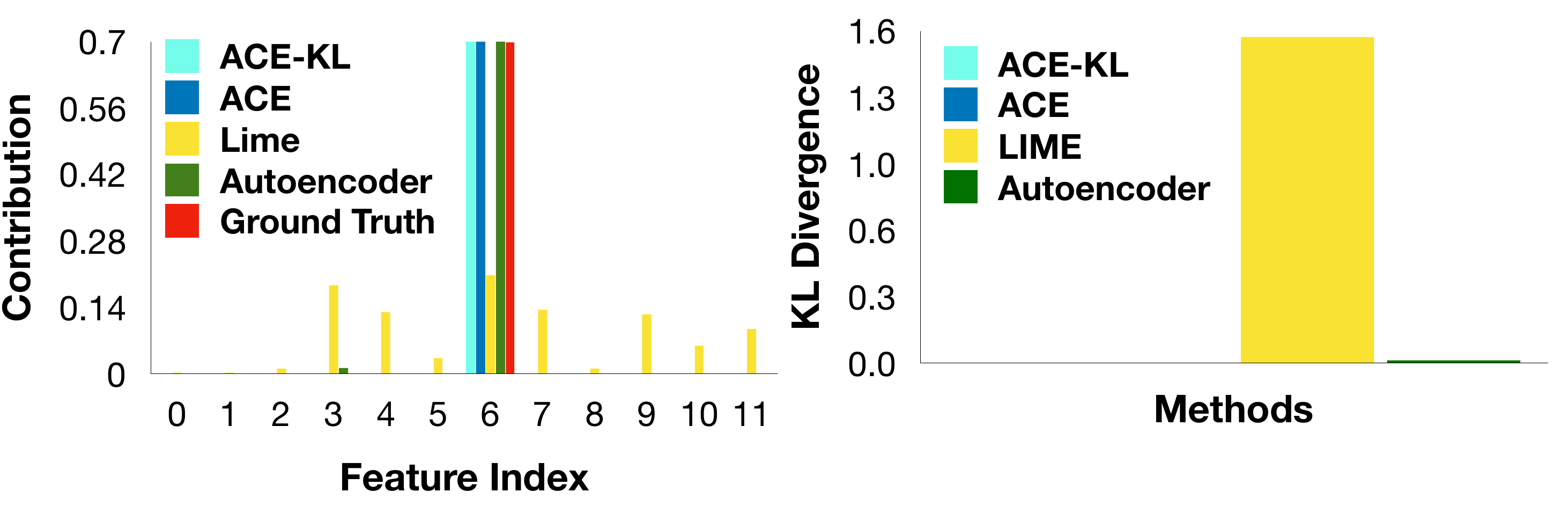}
        \caption{Features 6 is perturbed.}
    \end{subfigure}
    \begin{subfigure}[b]{0.48\textwidth}
        \includegraphics[width=\textwidth]{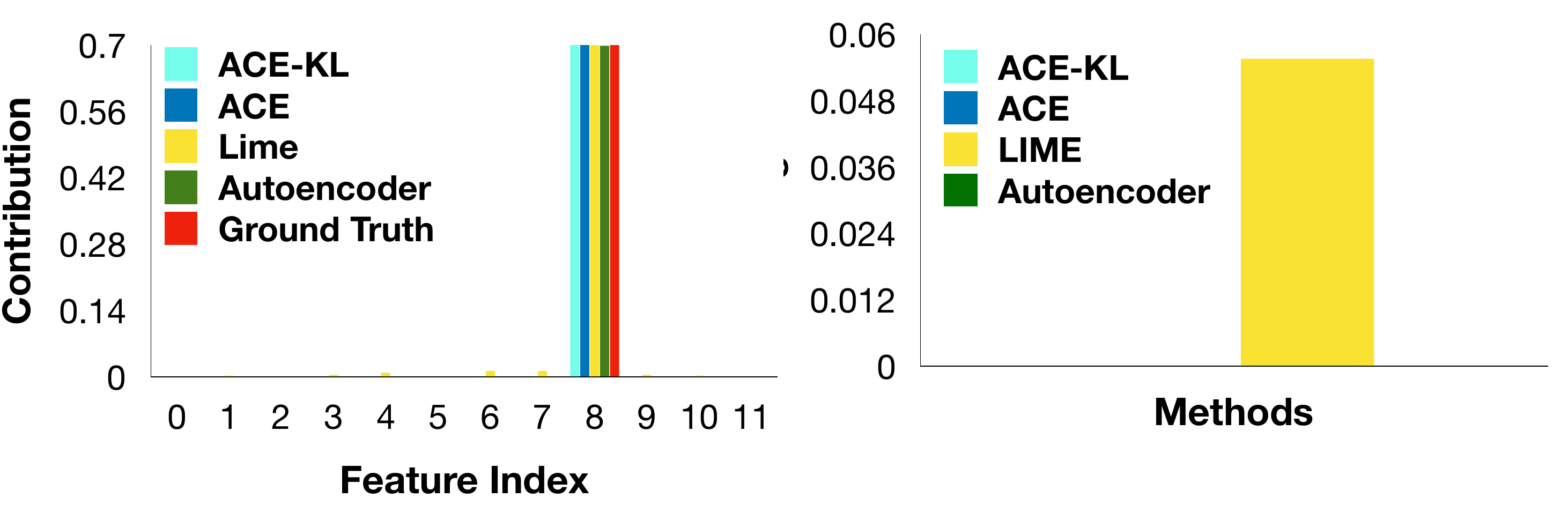}
        \caption{Features 8 is perturbed.}
    \end{subfigure}
    \begin{subfigure}[b]{0.48\textwidth}
        \includegraphics[width=\textwidth]{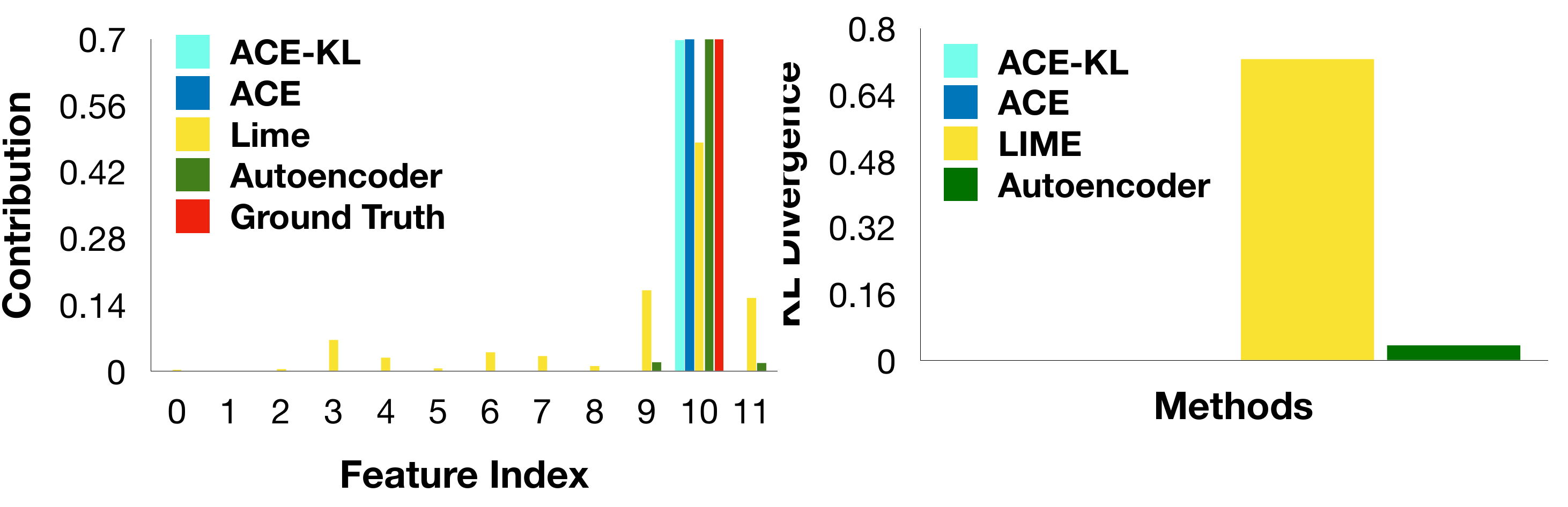}
        \caption{Features 10 is perturbed.}
    \end{subfigure}
    \caption{Feature contribution calculated by different methods on 6 synthetic examples, where each of them has one feature perturbed. The left side are the contributions of each feature calculated using different method, and the right side are the KL-divergence for each method.}
\label{fig:onepert}
\vspace{-1em}
\end{figure} 

\subsection{Perturb Two Feature for CERT}
In this section, we present the experimental results of applying different anomaly explanation methods by perturbing two of the twelve features in the CERT data set as an injected anomaly in Figure \ref{fig:twopert}. As previously described, the left column shows the contributions of each feature calculated by anomaly explanation method, and the right column the KL-divergence between the distribution of the calculated contributions and the true distribution. As seen from Figure \ref{fig:twopert}, ACE and ACE-KL perform well across all four examples consistently, while Autoencoder only captures the second anomaly in the third and the fourth examples, and LIME fails to capture any of the anomalies in an accurate manner, with higher KL-divergence compared to the true distribution. These results further empirically support our claim that LIME is not suitable for anomaly explanation in the security domain while ACE and ACE-KL are very powerful tools in this application domain.

\begin{figure}[!ht]
    \centering
    \begin{subfigure}[b]{0.48\textwidth}
        \includegraphics[width=\textwidth]{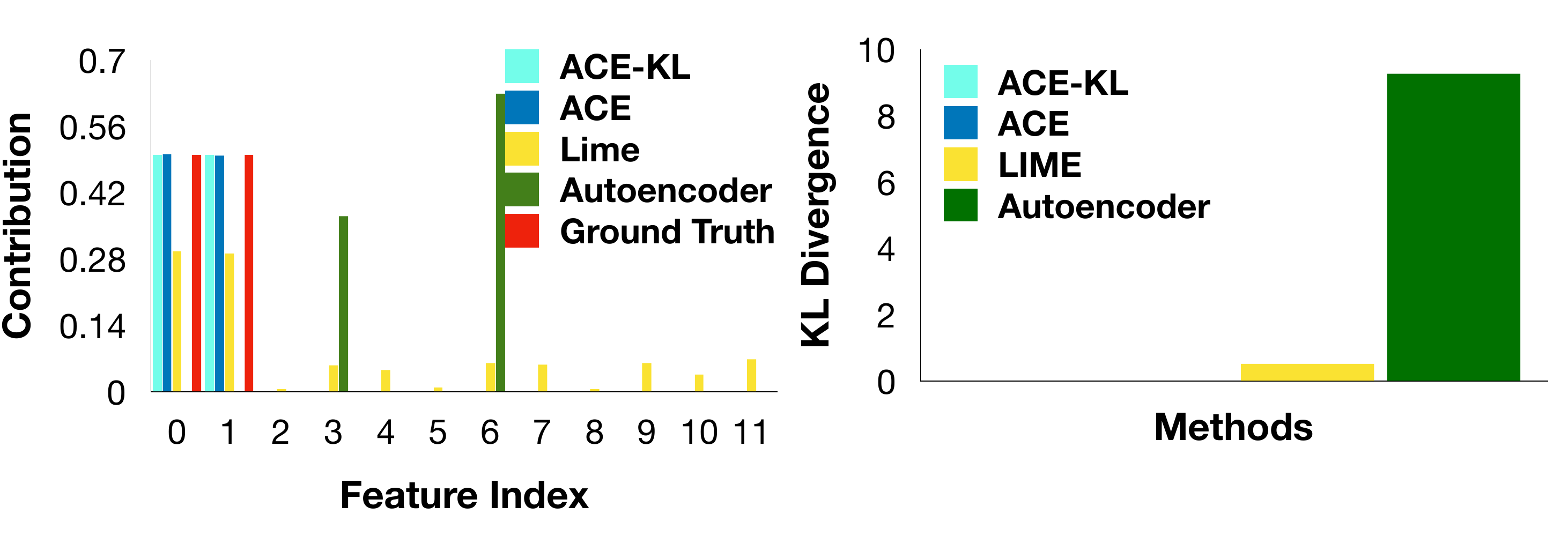}
        \caption{Features 0, 1 are perturbed.}
    \end{subfigure}
    \begin{subfigure}[b]{0.48\textwidth}
        \includegraphics[width=\textwidth]{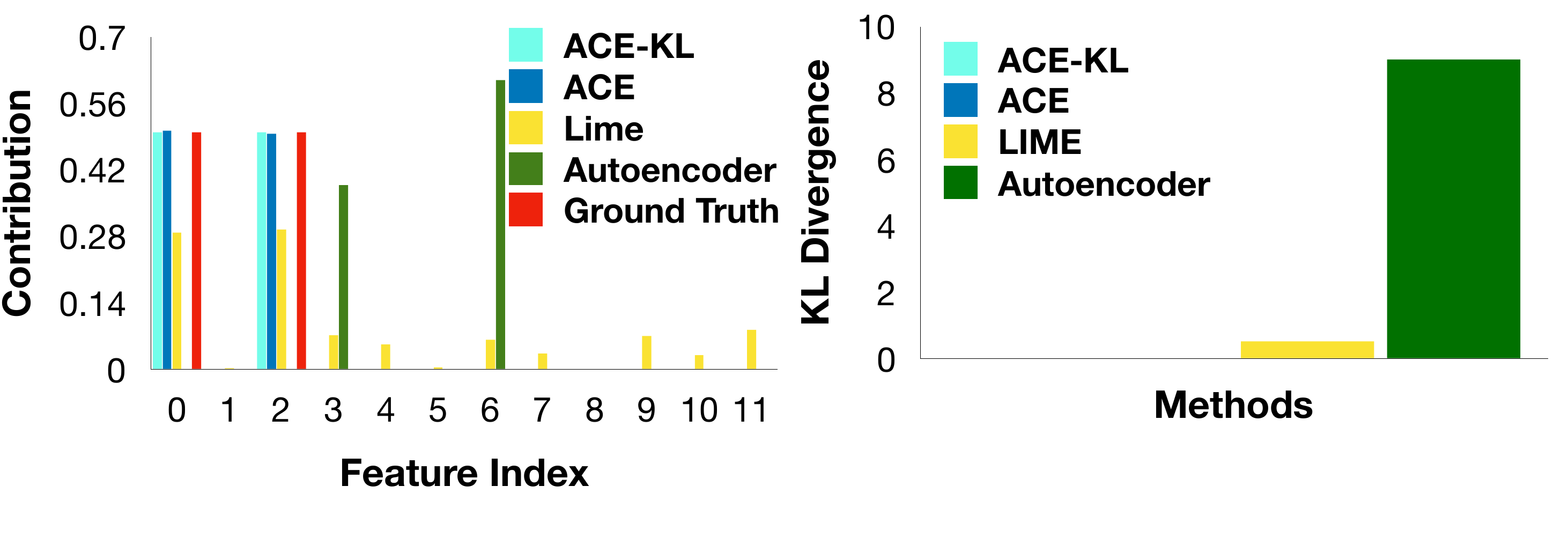}
        \caption{Features 0, 2 are perturbed.}
    \end{subfigure}
    \begin{subfigure}[b]{0.48\textwidth}
        \includegraphics[width=\textwidth]{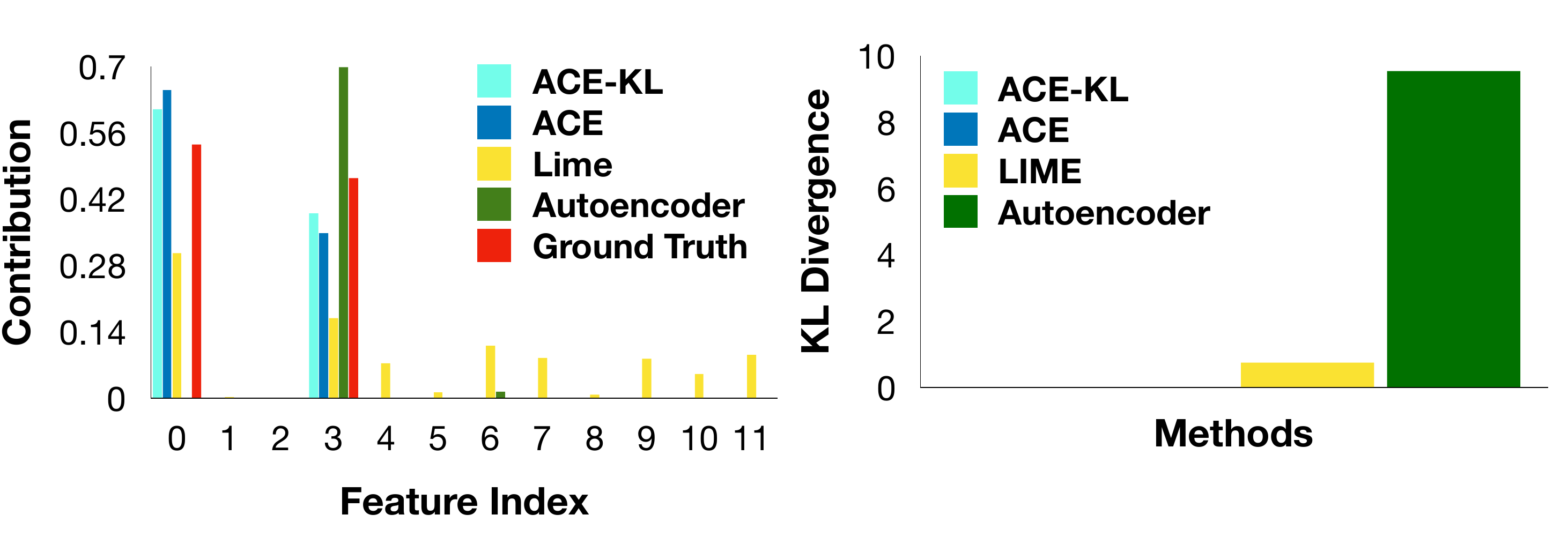}
        \caption{Features 0, 3 are perturbed.}
    \end{subfigure}
    \begin{subfigure}[b]{0.48\textwidth}
        \includegraphics[width=\textwidth]{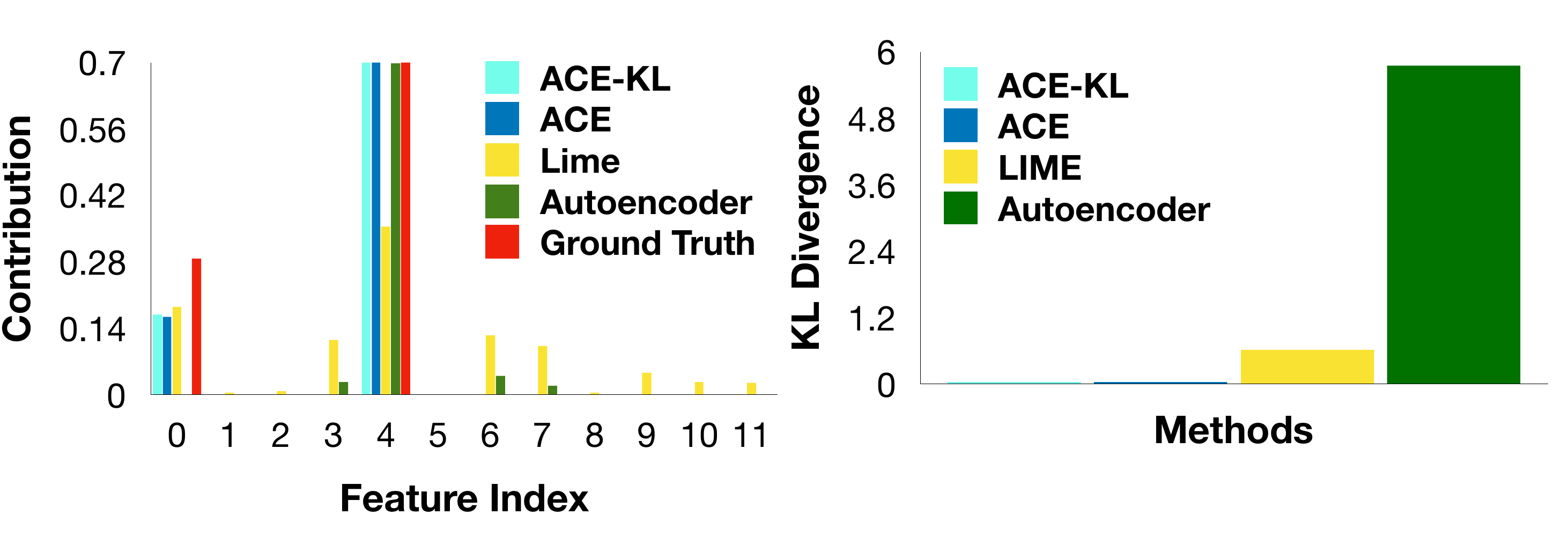}
        \caption{Features 0, 4 are perturbed.}
    \end{subfigure}
    \caption{Feature contribution calculated by different methods on 4 synthetic examples, where each of them has two features perturbed. The left side are the contributions of each feature calculated using different method, and the right side are the KL-divergence for each method.}
\label{fig:twopert}
\vspace{-1em}
\end{figure}

%% file: main.bbl
\begin{thebibliography}{10}
\providecommand{\url}[1]{#1}
\csname url@samestyle\endcsname
\providecommand{\newblock}{\relax}
\providecommand{\bibinfo}[2]{#2}
\providecommand{\BIBentrySTDinterwordspacing}{\spaceskip=0pt\relax}
\providecommand{\BIBentryALTinterwordstretchfactor}{4}
\providecommand{\BIBentryALTinterwordspacing}{\spaceskip=\fontdimen2\font plus
\BIBentryALTinterwordstretchfactor\fontdimen3\font minus
  \fontdimen4\font\relax}
\providecommand{\BIBforeignlanguage}[2]{{%
\expandafter\ifx\csname l@#1\endcsname\relax
\typeout{** WARNING: IEEEtran.bst: No hyphenation pattern has been}%
\typeout{** loaded for the language `#1'. Using the pattern for}%
\typeout{** the default language instead.}%
\else
\language=\csname l@#1\endcsname
\fi
#2}}
\providecommand{\BIBdecl}{\relax}
\BIBdecl

\bibitem{whitehousereport}
\BIBentryALTinterwordspacing
(2018) The cost of malicious cyber activity to the u.s. economy. [Online].
  Available:
  \url{www.whitehouse.gov/wp-content/uploads/2018/03/The-Cost-of-Malicious-Cyber-Activity-to-the-U.S.-Economy.pdf}
\BIBentrySTDinterwordspacing

\bibitem{forbesreport}
\BIBentryALTinterwordspacing
(2018) Global information security spending to exceed \$124b in 2019. [Online].
  Available:
  \url{www.forbes.com/sites/rogeraitken/2018/08/19/global-information-security-spending-to-exceed-124b-in-2019-privacy-concerns-driving-demand}
\BIBentrySTDinterwordspacing

\bibitem{Veeramachaneni2016}
K.~Veeramachaneni, I.~Arnaldo, V.~Korrapati, C.~Bassias, and K.~Li,
  ``\textit{AI}$^2$: Training a big data machine to defend,'' in
  \emph{BigDataSecurity, HPSC and IDS}, 2016.

\bibitem{Caruana2015}
R.~Caruana, Y.~Lou, J.~Gehrke, P.~Koch, M.~Sturm, and N.~Elhadad,
  ``Intelligible models for healthcare: Predicting pneumonia risk and hospital
  30-day readmission,'' in \emph{Proc. of the International Conference on
  Knowledge Discovery and Data Mining (KDD)}, 2015.

\bibitem{Shi2012}
Y.~Shi, ``China's national personal credit scoring system: A real-life
  intelligent knowledge application,'' in \emph{Proc. of the International
  Conference on Knowledge Discovery and Data Mining (KDD)}, 2012.

\bibitem{Tuor2017}
A.~Tuor, S.~Kaplan, B.~Hutchinson, N.~Nichols, and S.~Robinson, ``Deep learning
  for unsupervised insider threat detection in structured cybersecurity data
  streams,'' in \emph{AAAI Workshop on AI for Cybersecurity Workshop}, 2017.

\bibitem{yousefi2017autoencoder}
M.~Yousefi-Azar, V.~Varadharajan, L.~Hamey, and U.~Tupakula,
  ``Autoencoder-based feature learning for cyber security applications,'' in
  \emph{2017 International joint conference on neural networks (IJCNN)}.\hskip
  1em plus 0.5em minus 0.4em\relax IEEE, 2017, pp. 3854--3861.

\bibitem{berman2019survey}
D.~S. Berman, A.~L. Buczak, J.~S. Chavis, and C.~L. Corbett, ``A survey of deep
  learning methods for cyber security,'' \emph{Information}, vol.~10, no.~4, p.
  122, 2019.

\bibitem{cui2018detection}
Z.~Cui, F.~Xue, X.~Cai, Y.~Cao, G.-g. Wang, and J.~Chen, ``Detection of
  malicious code variants based on deep learning,'' \emph{IEEE Transactions on
  Industrial Informatics}, vol.~14, no.~7, pp. 3187--3196, 2018.

\bibitem{dls}
\BIBentryALTinterwordspacing
(2019) Deep learning and security workshop. [Online]. Available:
  \url{www.ieee-security.org/TC/SPW2019/DLS/}
\BIBentrySTDinterwordspacing

\bibitem{Ribeiro2016a}
M.~T. Ribeiro, S.~Singh, and C.~Guestrin, ``"why should i trust you?":
  Explaining the predictions of any classifier,'' in \emph{Proc. of the
  International Conference on Knowledge Discovery and Data Mining (KDD)}, 2016.

\bibitem{biranSurvey2017}
O.~Biran and C.~Cotton, ``Explanation and justification in machine learning: A
  survey,'' in \emph{IJCAI-17 Workshop on Explainable AI (XAI)}, 2017.

\bibitem{doshi2017}
F.~Doshi-Velez and B.~Kim, ``Towards a rigorous science of interpretable
  machine learning,'' \emph{ArXiv e-prints}, 2017.

\bibitem{Lipton2018}
Z.~C. Lipton, ``The mythos of model interpretability,'' \emph{Queue}, 2018.

\bibitem{GuidottiSurvey2018}
R.~Guidotti, A.~Monreale, S.~Ruggieri, F.~Turini, F.~Giannotti, and
  D.~Pedreschi, ``A survey of methods for explaining black box models,''
  \emph{CSUR}, no.~5, Aug. 2018.

\bibitem{suermondt1992}
H.~J. Suermondt, ``Explanation in bayesian belief networks,'' Ph.D.
  dissertation, Stanford University, 1992.

\bibitem{feraud2002methodology}
R.~F{\'e}raud and F.~Cl{\'e}rot, ``A methodology to explain neural network
  classification,'' \emph{Neural Networks}, 2002.

\bibitem{robnik2011}
M.~Robnik-{\v{S}}ikonja, A.~Likas, C.~Constantinopoulos, I.~Kononenko, and
  E.~{\v{S}}trumbelj, ``Efficiently explaining decisions of probabilistic rbf
  classification networks,'' in \emph{International Conference on Adaptive and
  Natural Computing Algorithms}, 2011.

\bibitem{landecker2013}
W.~Landecker, M.~D. Thomure, L.~M. Bettencourt, M.~Mitchell, G.~T. Kenyon, and
  S.~P. Brumby, ``Interpreting individual classifications of hierarchical
  networks,'' in \emph{IEEE Symposium on CIDM}, 2013.

\bibitem{martens2008}
D.~Martens, J.~Huysmans, R.~Setiono, J.~Vanthienen, and B.~Baesens, ``Rule
  extraction from support vector machines: An overview of issues and
  application in credit scoring,'' in \emph{Rule extraction from support vector
  machines}.\hskip 1em plus 0.5em minus 0.4em\relax Springer, 2008.

\bibitem{robnik2008}
M.~Robnik-{\v{S}}ikonja and I.~Kononenko, ``Explaining classifications for
  individual instances,'' \emph{IEEE Transactions on Knowledge and Data
  Engineering}, 2008.

\bibitem{kononenko2013}
I.~Kononenko, E.~{\v{S}}trumbelj, Z.~Bosni{\'c}, D.~Pevec, M.~Kukar, and
  M.~Robnik-{\v{S}}ikonja, ``Explanation and reliability of individual
  predictions,'' \emph{Informatica}, 2013.

\bibitem{baehrens2010}
D.~Baehrens, T.~Schroeter, S.~Harmeling, M.~Kawanabe, K.~Hansen, and K.-R.
  M{\~A}{\v{z}}ller, ``How to explain individual classification decisions,''
  \emph{JMLR}, 2010.

\bibitem{Poulin2006}
B.~Poulin, D.~Eisner, Roman~andSzafron, P.~Lu, R.~Greiner, D.~S~Wishart,
  A.~Fyshe, B.~Pearcy, C.~MacDonell, and J.~Anvik, ``Visual explanation of
  evidence with additive classifiers,'' in \emph{National Conference on
  Artificial Intelligence}, 2006.

\bibitem{Strumbelj2010}
E.~Strumbelj and I.~Kononenko, ``An efficient explanation of individual
  classifications using game theory,'' \emph{JMLR}, 2010.

\bibitem{Vidovic2016}
M.~M.~C. {Vidovic}, N.~{G{\"o}rnitz}, K.-R. {M{\"u}ller}, and M.~{Kloft},
  ``{Feature Importance Measure for Non-linear Learning Algorithms},''
  \emph{ArXiv e-prints}, 2016.

\bibitem{Guidotti2018}
R.~{Guidotti}, A.~{Monreale}, S.~{Ruggieri}, D.~{Pedreschi}, F.~{Turini}, and
  F.~{Giannotti}, ``{Local Rule-Based Explanations of Black Box Decision
  Systems},'' \emph{ArXiv e-prints}, 2018.

\bibitem{Lundberg2017}
S.~M. Lundberg and S.-I. Lee, ``A unified approach to interpreting model
  predictions,'' in \emph{Proc. of the Conference on Advances in Neural
  Information Processing Systems (NIPS)}, 2017.

\bibitem{chandola2009}
V.~Chandola, A.~Banerjee, and V.~Kumar, ``Anomaly detection: A survey,''
  \emph{CSUR}, 2009.

\bibitem{Hara2015}
S.~Hara, T.~Morimura, T.~Takahashi, H.~Yanagisawa, and T.~Suzuki, ``A
  consistent method for graph based anomaly localization,'' in \emph{Proc. of
  the International Conference on Artificial Intelligence and Statistics
  (AISTATS)}, 2015.

\bibitem{Jiang2011}
R.~Jiang, H.~Fei, and J.~Huan, ``Anomaly localization for network data streams
  with graph joint sparse pca,'' in \emph{Proc. of the International Conference
  on Knowledge Discovery and Data Mining (KDD)}, 2011.

\bibitem{Hara2017}
S.~Hara, T.~Katsuki, H.~Yanagisawa, T.~Ono, R.~Okamoto, and S.~Takeuchi,
  ``Consistent and efficient nonparametric different-feature selection,'' in
  \emph{Proc. of the International Conference on Artificial Intelligence and
  Statistics (AISTATS)}, 2017.

\bibitem{Woodall2003}
W.~H. Woodall, R.~Koudelik, K.-L. Tsui, S.~B. Kim, Z.~G. Stoumbos, and C.~P.~C.
  MD, ``A review and analysis of the mahalanobis—taguchi system,''
  \emph{Technometrics}, 2003.

\bibitem{Hirose2009}
S.~Hirose, K.~Yamanishi, T.~Nakata, and R.~Fujimaki, ``Network anomaly
  detection based on eigen equation compression,'' in \emph{Proc. of the
  International Conference on Knowledge Discovery and Data Mining (KDD)}, 2009.

\bibitem{Tsuyoshi2009}
T.~Id{\'e}, A.~C. Lozano, N.~Abe, and Y.~Liu, ``Proximity-based anomaly
  detection using sparse structure learning,'' in \emph{SDM}, 2009.

\bibitem{Dugas2000}
C.~Dugas, Y.~Bengio, F.~B{\'e}lisle, C.~Nadeau, and R.~Garcia, ``Incorporating
  second-order functional knowledge for better option pricing,'' in \emph{Proc.
  of the Conference on Advances in Neural Information Processing Systems
  (NIPS)}, 2000.

\bibitem{Boyd2004}
S.~Boyd and L.~Vandenberghe, \emph{Convex Optimization}.\hskip 1em plus 0.5em
  minus 0.4em\relax New York, NY, USA: Cambridge University Press, 2004.

\bibitem{lindauer2014}
B.~Lindauer, J.~Glasser, M.~Rosen, K.~C. Wallnau, and L.~ExactData,
  ``Generating test data for insider threat detectors.'' \emph{Journal of
  Wireless Mobile Networks, Ubiquitous Computing, and Dependable Applications},
  2014.

\bibitem{glasser2013}
J.~Glasser and B.~Lindauer, ``Bridging the gap: A pragmatic approach to
  generating insider threat data,'' in \emph{IEEE SPW}, 2013.

\bibitem{Shiravi2012}
A.~Shiravi, H.~Shiravi, M.~Tavallaee, and A.~A. Ghorbani, ``Toward developing a
  systematic approach to generate benchmark datasets for intrusion detection,''
  \emph{Computer Security}, 2012.

\bibitem{Zhou2012}
Y.~Zhou and X.~Jiang, ``Dissecting android malware: Characterization and
  evolution,'' in \emph{IEEE Symposium on SP}, 2012.

\bibitem{Peng2012}
H.~Peng, C.~Gates, B.~Sarma, N.~Li, Y.~Qi, R.~Potharaju, C.~Nita-Rotaru, and
  I.~Molloy, ``Using probabilistic generative models for ranking risks of
  android apps,'' in \emph{CCS}, 2012.

\end{thebibliography}
